\documentclass[accepted]{uai2022}
\usepackage[american]{babel}
\usepackage{natbib}
\usepackage{thmtools}
\usepackage{mathtools}
\usepackage{booktabs}
\usepackage{tikz}
\usetikzlibrary{positioning}
\usepackage{pgfplotstable}
\usepackage{pgfplots}
\pgfplotsset{compat=1.8}
\usepgfplotslibrary{groupplots}
\usepackage{amssymb}
\usepackage{amsmath}
\usepackage{xcolor}
\usepackage{float}
\usepackage{subfig}
\usepackage[linesnumbered,ruled,vlined]{algorithm2e}
\SetKwRepeat{Do}{do}{while}
\extrafloats{100}

\let\mc\mathcal
\let\mb\mathbf
\let\mt\mathtt
\let\mf\mathfrak
\let\bs\boldsymbol

\newcommand{\Prob}[0]{\mc{P}}
\newcommand{\G}[0]{\mc{G}}
\newcommand{\A}[0]{\mathsf{A}}
\newcommand{\I}[0]{\mt{I}}
\newcommand{\E}[0]{\mt{E}}  
\newcommand{\D}[0]{\mc{D}}
\newcommand{\Pa}[0]{\mt{Pa}}
\newcommand{\An}[0]{\mt{An}}
\newcommand{\Ch}[0]{\mt{Ch}}
\newcommand{\De}[0]{\mt{De}}
\newcommand{\Nd}[0]{\mt{Nd}}
\newcommand{\Pre}[0]{\mt{Pre}}
\newcommand{\MB}[0]{\mt{MB}}
\newcommand{\MEC}[0]{\mt{MEC}}
\newcommand{\CMC}[0]{\mt{CMC}}
\newcommand{\CFC}[0]{\mt{CFC}}
\newcommand{\SGS}[0]{\mt{SGS}}
\newcommand{\Pm}[0]{\mt{Pm}}
\newcommand{\uPm}[0]{\mt{uPm}}
\newcommand{\Fr}[0]{\mt{Fr}}
\newcommand{\uFr}[0]{\mt{uFr}}
\newcommand{\BIC}[0]{\mt{BIC}}
\newcommand{\DAG}[0]{\mt{DAG}}
\newcommand{\la}[0]{\langle}
\newcommand{\ra}[0]{\rangle}
\newcommand{\CI}[0]{\perp\!\!\!\perp}
\newcommand{\nCI}[0]{\perp\!\!\!\!/\!\!\!\!\!\perp}
\newcommand{\ot}[0]{\leftarrow}
\newcommand{\TP}[0]{\mathit{TP}}
\newcommand{\FP}[0]{\mathit{FP}}
\newcommand{\FN}[0]{\mathit{FN}}

\newtheorem{definition}{Definition}[section]
\newtheorem{corollary}[definition]{Corollary}
\newtheorem{lemma}[definition]{Lemma}

\newtheorem{theorem}[definition] {Theorem}

\newenvironment{proof}{\textit{Proof.\hspace{0.1cm}}}{\hfill$\square$}

\title{Greedy Relaxations of the Sparsest Permutation Algorithm}

\author[1]{\href{mailto:waiyinl@andrew.cmu.edu?Subject=}{Wai-Yin~Lam}}
\author[1]{\href{mailto:bjandrew@andrew.cmu.edu}{Bryan~Andrews}}
\author[1]{\href{mailto:jdramsey@andrew.cmu.edu}{Joseph~Ramsey}}

\affil[1]{
    Department of Philosophy\\
    Carnegie Mellon University\\
    Pittsburgh, Pennsylvania, USA
}
  
\begin{document}
\maketitle

\begin{abstract}
There has been an increasing interest in methods that exploit permutation reasoning to search for directed acyclic causal models, including the ``Ordering Search’' of Teyssier and Kohler and GSP of Solus, Wang and Uhler. We extend the methods of the latter by a permutation-based operation \textit{tuck}, and develop a class of algorithms, namely GRaSP, that are efficient and pointwise consistent under increasingly weaker assumptions than faithfulness. The most relaxed form of GRaSP outperforms many state-of-the-art causal search algorithms in simulation, allowing efficient and accurate search even for dense graphs and graphs with more than 100 variables.
\end{abstract}

\section{Introduction}
\label{sec:intro}
Searching for causal models by identifying patterns of conditional independence in observational data has become a well-established activity, though it is not without detractors. For one thing, it is commonly believed that the only correct method for establishing causal relationships is through experimental manipulation, as is done in a randomized controlled trial. Accordingly, causal inference from observational data alone can be seen as second-rate. This is not completely unreasonable; many causal search algorithms, even in seemingly ideal conditions, demonstrate poor performance, calling into question whether their inferences can be relied upon. Furthermore, the theoretical assumptions made by these algorithms are often criticized for being too strong. More specifically, these algorithms assume that the true model belongs to a model class with no latent variables or no cycles, and that the patterns of conditional independence in the data generating distribution can be represented by the assumed model class exactly. The latter assumption is called the causal faithfulness condition and can be violated (or almost violated) by unexpected patterns of conditional independence that arise from subtleties in the distribution, such as (near) determinism or (almost) path cancellation.

The most common model class assumed by causal search algorithms is characterized by directed acyclic graphs (DAGs). Many algorithms for causal inference search the space of causal DAGs, such as the PC (``Peter and Clark'', \citep{spirtes2000causation}) and GES (``Greedy equivalence Search'', \citep{chickering2002optimal}) algorithms, and provably return a set of DAGs that contains the true model under the causal faithfulness condition. However, the performance of these algorithms appear to fall short of their theoretical claims, especially when the true model is densely connected. One hypothesis for this phenomenon is that almost-violations of the causal faithfulness condition frequently occur and impede search procedures \citep{uhler2013geometry}. Accordingly, the performance of these algorithms might be improved by relaxing the causal faithfulness condition, as is done by the SP (``Sparsest Permutation'') algorithm of \cite{raskutti2018learning}. SP considers the space of variable orderings and builds a DAG using a procedure inspired by \cite{verma1988causal}, where the parents of each variable are selected from the preceding variables in the permutation. Ultimately, the permutations that induce DAGs with the minimal edge count are selected.

\cite{raskutti2018learning} proved that if the data generating distribution is a graphoid, then the set of DAGs returned by SP contains the true model under an assumption strictly weaker than the causal faithfulness condition. While SP recovers the set of all frugal models, it is super exponential in the number of variables, that is, if there are $n$ variables, then there are $n!$ permutations that must be visited. In practice, it is limited to a maximum of about nine variables due to its computational complexity. This naturally raises the question: is there an algorithm that is equally accurate in most cases for such data, but that can scale to larger problems?

\cite{teyssier2005ordering} give a clever search and score procedure, ``Ordering Search'', over variable permutations, pointing out that when two adjacent variables in a permutation are swapped, only local scores for the swapped variables need to be recalculated, the rest of the score calculation remains unchanged---this swapping operation is called an adjacency transposition (AT). The Ordering Search algorithm greedily traverses the space of permutations with adjacency transpositions using a hill-climbing approach, random restarts, and a tabu list. However, they do not give any consistency guarantees. 

The ESP (``Edge Sparsest Permutation'') algorithm of \cite{solus2021consistency} iterates upon the Ordering Search algorithm by greedily traversing the space of permutations by sequences of ATs where each AT leads to an equal or less edge count, found by depth first search (DFS), to achieve asymptotic correctness. In addition, their TSP (``Triangle Sparsest Permutation'') algorithm uses the theory of \cite{chickering2002optimal} to navigate the space of DAGs, more efficiently than ESP, under a stronger assumption. A simulation study using a Python implementation of TSP \citep{solus2021consistency} suggests that this procedure is fast, but has difficulty scaling accurately to moderate or large sized graphs \citep{lu2021improving}.

To address the scaling problem for both accuracy and timing, in this paper we explore different ways of traversing the space of permutations that get closer to the performance and assumption relaxation of Raskutti and Uhler while maintaining scalability. As part of this effort, we also use the ``Grow-Shrink'' algorithm from \cite{margaritis1999bayesian} to learn the DAG.

In what follows, we give an elaboration of the theoretical background of our set of permutation-based procedures, GRaSP (``Greedy Relaxations of Sparsest Permutation''). GRaSP has three tiers, GRaSP$_0$ (basically equivalent to TSP), GRaSP$_1$ (basically equivalent to ESP), and GRaSP$_2$ (a novel relaxation); we show how moving from a lower tier to a higher tier results in a gradual theoretical relaxation of the permutation search space and thus an improvement in accuracy. We then follow this with a study of oracle behavior for GRaSP$_0$, GRaSP$_1$, and GRaSP$_2$ on exhaustive lists of independence models with violations of faithfulness for all 4-variable regular Gaussian and positive discrete distributions and all 5-variable unfaithful DAGs with added marginal independencies between a pair of variables. We also give a detailed simulation study for the linear, Gaussian case for larger possibly dense models of up to 100 variables, with consistently accurate results using GRaSP$_2$. Further, we study an empirical example to test GRaSP$_2$. We then give a conclusion and discussion where we point out areas of immediate future work.


\section{Contributions}
\label{sec:contributions}
The most salient contribution is that GRaSP$_2$ can scale to at least 100 variables with average degree at least 10 on a laptop with high adjacency and arrowhead precision and recall for the linear, Gaussian case, addressing the longstanding practical problem of dense graph causal search in a meaningful way.

Second, theoretical development over assumptions on causal discovery from previous works has been simplified, in places corrected, and reworked as a structured study of casual razors. Accordingly, the proof that GRaSP$_0$, TSP, and by implication GSP, require faithfulness is a logical discovery. Also, the proof that faithfulness is equivalent to unique Pearl-minimality is a novel contribution.

Third, we extended the discussion of unit tests initiated in \citep{solus2021consistency} considerably, using the criterion that a wide variety of unit tests should systematically pass on all initial permutations using a d-separation oracle. More specifically, we run GRaSP on models detailed in \cite{vsimecek2006gaussian, vsimecek2006short} and those listed in Appendix \textcolor{blue}{\textbf{\ref{app:unit_tests}}}.

Finally, the tuck operation is a novel transformation that has not been considered in the literature before. We show that traversing any edge of the DAG-associahedron can be equivalently done via a tuck. Reframing TSP in terms of the tuck operation allows TSP and ESP to be neatly placed into a hierarchy. Moreover, it admits the natural generalization to GRaSP$_2$ (by not restricting which edges can be tucked).

\section{Background}
\label{sec:background}
Throughout this paper, italicized letters are used to denote variables (e.g., $X_1, Y$) and boldfaced letters for sets of variables (e.g., $\mb{X}$). Graphical definitions and notations related to directed acyclic graphs (DAGs) are provided in Appendix \textcolor{blue}{\textbf{\ref{app:graph_def}}}. A DAG $\G$ over a set of measured variables $\mb{V} = \{X_1,..., X_m\}$ consists of $m$ vertices $\mb{v} = \{1,...,m\}$ where each vertex $i$ associates to the variable $X_i$, and each directed edge between two distinct vertices $j \to k$ represents the direct causal influence from $X_j$ to $X_k$. We write $\mb{i} \perp_\G \mb{j}\,|\,\mb{k}$ to denote the \textit{d-separation} relation between $\mb{i}$ and $\mb{j}$ given $\mb{k}$ in $\G$ for any pairwise disjoint subsets of vertices $\mb{i}, \mb{j}, \mb{k} \subseteq \mb{v}$. Similarly, given a joint probability distribution $\Prob$ over $\mb{V}$, denote $\mb{X} \CI_\Prob \mb{Y}\,|\,\mb{Z}$ as the \textit{conditional independence} (CI) relation between $\mb{X}$ and $\mb{Y}$ given $\mb{Z}$ for any pairwise disjoint subsets of variables $\mb{X}, \mb{Y}, \mb{Z} \subseteq \mb{V}$. 

A \textit{model} is a pair $(\G, \Prob)$ where $\G$ is a DAG and $\Prob$ is a joint probability distribution over the same set of measured variables $\mb{V}$. We use $\G^*$ to refer to the \textit{true} data-generating DAG such that $(\G^*, \Prob)$ is the true model assumed to always exist. Certain standard properties of a model can be defined in terms of the d-separation relations in $\G$ and the CI relations in $\Prob$. Denote $\I(\G) = \{\la \mb{X}_\mb{j}, \mb{X}_\mb{k}\,|\,\mb{X}_\mb{l}\ra: \mb{j} \perp_\G \mb{k}\,|\,\mb{l}\}$ where $\mb{X}_\mb{i} = \{X_j \in \mb{V}: j \in \mb{i}\}$ for every $\mb{i} \subseteq \mb{v}$, and $\I(\Prob) = \{\la \mb{X}, \mb{Y}\,|\,\mb{Z}\ra: \mb{X} \CI_\Prob \mb{Y}\,|\,\mb{Z}\}$. Let $\DAG(\mb{V})$ be the set of all possible DAGs over $\mb{V}$.

\begin{definition}
\label{Markov}
(Markov) For any joint probability distribution $\Prob$ over $\mb{V}$, define $\CMC(\Prob) = \{\G \in \DAG(\mb{V}): \I(\G) \subseteq \I(\Prob)\}$ as the set of Markovian DAGs. $(\G^*, \Prob)$ satisfies the Markov assumption if $\G^* \in \CMC(\Prob)$.
\end{definition}

\begin{definition}
\label{faithful}
(Faithfulness) For any joint probability distribution $\Prob$, define $\CFC(\Prob) = \{\G \in \CMC(\Prob): \I(\Prob) \subseteq \I(\G)\}$ as the set of faithful DAGs. $(\G^*, \Prob)$ satisfies the faithfulness assumption if $\G^* \in \CFC(\Prob)$.
\end{definition}

A causal search algorithm is a procedure of recovering the causal information of the true DAG from its underlying joint probability distribution. Let $\MEC(\G)$ be the \textit{Markov equivalence class} (MEC) of $\G$ such that $\I(\G) = \I(\G')$ for each $\G' \in \MEC(\G)$. One crucial goal of causal search is the identification of $\MEC(\G^*)$ from $\Prob$. With regard to this goal, a causal search algorithm is \textit{correct} if its output DAG (or the DAG induced by its output) is in $\MEC(\G^*)$. All known causal search algorithms assume the Markov assumption, and some well-known algorithms in the relevant literature (e.g., GES) assume faithfulness as well. Nevertheless, as pointed out by \cite{uhler2013geometry}, learning CI relations from data by hypothesis testing is error-prone, and almost-violations of faithfulness are common. This motivates the exploration of causal search algorithms which rely on assumptions strictly weaker than faithfulness. These assumptions, faithfulness included, are what we refer to as \textit{causal razors}.

One recent approach proposed by \cite{raskutti2018learning} is the \textit{SP} algorithm, which identifies the set of \textit{sparsest permutations} defined over $\mb{v}$ under the following causal razor. Let $\E(\G)$ be the set of directed edges in a DAG $\G$. 

\begin{definition}
\label{frugal}
(U-frugality) For any joint probability distribution $\Prob$, define $\Fr(\Prob) = \{\G \in \CMC(\Prob): \neg \exists \G' \in \CMC(\Prob)$ s.t. $|\E(\G')| < |\E(\G)|\}$ and $\uFr(\Prob) = \{\G \in \Fr(\Prob): \neg \exists \G' \in \Fr(\Prob)$ s.t. $\G' \notin \MEC(\G)\}$ as the sets of frugal DAGs and uniquely frugal, or u-frugal, DAGs respectively. $(\G^*, \Prob)$ satisfies the u-frugality assumption if $\G^* \in \uFr(\Prob)$.\footnote{This assumption is named as \textit{sparsest Markov representation} (SMR) in \citep{raskutti2018learning}.}
\end{definition}

In words, u-frugality requires that $\G^*$ is not only the sparsest Markovian DAG, but also that all sparsest Markovian DAGs belong to the same MEC as $\G^*$. \cite{raskutti2018learning} showed that SP is correct under u-frugality which is strictly weaker than faithfulness. Below we introduce some necessary notations of permutation-based algorithms. To begin with, we refer the readers to Appendix \textcolor{blue}{\textbf{\ref{app:graphoid}}} for the \textit{graphoid axioms}. Generally speaking, every joint probability distribution is a \textit{semigraphoid}, strictly positive distributions are \textit{graphoids}, and regular Gaussian distributions are \textit{compositional graphoids}.  

Given $\mb{V} = \{X_1,..., X_m\}$, let $\Pi(\mb{v})$ be the set of all \textit{permutations} over $\mb{v} = \{1,...,m\}$. For each $\pi \in \Pi(\mb{v})$, let $\pi_i$ be the $i$-th vertex in $\pi$, $\pi[j]$ be the index of vertex $j$ in $\pi$ (s.t. $\pi_{\pi[j]} = j$), and $\Pre(j, \pi) = \{\pi_i: 1 \leq i < \pi[j]\}$ be the set of vertices that precede $j$'s index in $\pi$. We say that $\pi \in \Pi(\mb{v})$ is a \textit{causal order} of $\G \in \mt{DAG}(\mb{V})$ if $i \in \Pre(j, \pi)$ for each $j \in \mb{v}$ and each $i \in \An(j, \G)$ (i.e., the set of $j$'s \textit{ancestors} in $\G$). Given a graphoid $\Prob$ over $\mb{V}$, each $\pi \in \Pi(\mb{v})$ induces a DAG $\G_{\pi}$ satisfying the following condition:
\begin{align}
    & j \in \Pre(k, \pi) \text{ and } X_j \nCI_\Prob X_k\,|\,\mb{X}_{\Pre(k, \pi)\setminus \{j\}}\nonumber\\
    &\Leftrightarrow (j \to k) \in \E(\G_\pi). \tag{RU}
\end{align}

(RU) is the method of constructing a unique DAG from $\pi$ and $\Prob$ discussed in \citep{raskutti2018learning}. It is derived from a more general method in \citep{verma1988causal}. The two methods will be compared in Appendix \textcolor{blue}{\textbf{\ref{app:DAG_induce}}}. But we refer to $\G_\pi$ as the DAG induced from $\pi$ and the graphoid $\Prob$ using (RU) unless specified otherwise. Obviously, $\pi$ is a causal order of $\G_\pi$. Below is an important feature of $\G_\pi$.

\begin{definition}
\label{SGS-minimal} 
(SGS-minimality) For any joint probability distribution $\Prob$, define $\SGS(\Prob) = \{\G \in \CMC(\Prob): \neg \exists \G' \in \CMC(\Prob)$ s.t. $\E(\G') \subset \E(\G)\}$ as the set of SGS-minimal DAGs.\footnote{SGS-minimality is also known as \textit{minimal I-map} in the literature. We follow \cite{zhang2013comparison} to refer to this minimality condition as the one discussed in \citep{spirtes2000causation}.} 
\end{definition}

\begin{theorem}
\label{RU-theorem}
\citep{verma1988causal, raskutti2018learning} Given a graphoid $\Prob$ over $\mb{V}$, $\G_\pi$ induced by $\pi$ using (RU) is Markovian and SGS-minimal for every $\pi \in \Pi(\mb{v})$. 
\end{theorem}

The theorem above states that, for every permutation $\pi$, the induced DAG $\G_\pi$ is Markovian and no subgraph of $\G_\pi$ is Markovian. By identifying the sparsest permutation $\hat{\pi} = \text{argmin}_{\pi \in \Pi(\mb{v})}$ $|\E(\G_\pi)|$, $\G_{\hat{\pi}}$ returned by SP is guaranteed to be in $\MEC(\G^*)$ when u-frugality is satisfied. Nevertheless, SP needs to examine all $|\mb{v}|!$ permutations in $\Pi(\mb{v})$ to identify the sparsest one and hence lacks scalability. \cite{solus2021consistency} introduce a greedy version of SP, namely
\textit{Triangle SP} (TSP), which is proven to be correct under faithfulness.\footnote{In \citep{solus2021consistency}, \textit{Greedy SP} (GSP) is an operational version of TSP which imposes a depth bound on the DFS procedure and a parameter specifying the number of runs on selecting an arbitrary initial permutation. They claimed that TSP can be correct even when faithfulness fails. We examine their claim more carefully in Section \textcolor{blue}{\textbf{\ref{sec:method}}} and Appendix \textcolor{blue}{\textbf{\ref{app:ESP_GRaSP1}}}.} Below, we provide a quick and simple sketch of this result.

TSP borrows the \textit{Chickering algorithm} in \citep{chickering2002optimal} to perform their \textit{depth-first search} (DFS) procedure. For each vertex $i \in \mb{v}$, let $\Pa(i, \G)$ be the set of \textit{parents} in $\G$. A directed edge $j \to k$ is \textit{covered} in $\G$ if $\Pa(j, \G) = \Pa(k, \G) \setminus \{j\}$.

\begin{theorem}
\label{Chickering_seq} 
(Chickering sequences) \citep{chickering2002optimal} Given a set of variables $\mb{V}$, for every pair of DAGs $\G, \mc{H} \in \mt{DAG}(\mb{V})$, if $\I(\mc{H}) \subseteq \I(\G)$, there exists a sequence of DAGs, call it a \textit{Chickering sequence} $\la \mc{H} = \G^1, \G^2, ..., \G^k = \G\ra$ (from $\mc{H}$ to $\G$) s.t. $\I(\G^{i}) \subseteq \I(\G^{i+1})$ and $\G^{i+1}$ is obtained from $\G^i$ by either reversing a covered edge or deleting a directed edge for each $1 \leq i < k$.\footnote{The original theorem in \citep{chickering2002optimal} is expressed in terms of addition of directed edges. This modification helps by indicating that every Chickering sequence is a weakly decreasing sequence. In addition, one can easily observe that there does not exist any Chickering sequence from $\mc{H}$ to $\G$ if $\I(\mc{H}) \nsubseteq \I(\G)$.}
\end{theorem}

A sequence of DAGs $\la \G^1, ..., \G^k\ra$ is said to be \textit{weakly decreasing} if $|\E(\G^{i})| \geq |\E(\G^{i+1})|$ for each $1 \leq i < k$. Obviously, every Chickering sequence is weakly decreasing. Given an arbitrary initial permutation $\pi \in \Pi(\mb{v})$, TSP uses DFS to search for a Chickering sequence from $\G_{\pi}$ to some SGS-minimal DAG $\G_{\tau}$ where $|\E(\G_{\pi})| > |\E(\G_\tau)|$, and update $\G_\pi$ as $\G_\tau$ until no such $\G_\tau$ is found. Now we demonstrate TSP's correctness under faithfulness.

\begin{definition}
\label{P-minimal}
(U-P-minimality) For any joint probability distribution $\Prob$, define $\Pm(\Prob) = \{\G \in \CMC(\Prob): \neg \exists \G' \in \CMC(\Prob)$ s.t. $\I(\G) \subset \I(\G')\}$ and $\uPm(\Prob) = \{\G \in \Pm(\Prob): \neg \exists \G' \in \Pm(\Prob)$ s.t. $\G' \notin \MEC(\G)\}$ as the sets of P-minimal DAGs and uniquely P-minimal DAGs respectively. $(\G^*, \Prob)$ satisfies the u-P-minimality assumption if $\G^* \in \uPm(\Prob)$.\footnote{P-minimality refers to the minimality condition discussed in \citep{Pearl2009Causality}.}
\end{definition}

\begin{theorem}
\label{razors}
\citep{zhang2013comparison} For any joint probability distribution $\Prob$, $\CFC(\Prob) = \Pm(\Prob) = \MEC(\G^*)$ if faithfulness holds.
\end{theorem}

A DAG being P-minimal, as in \textbf{Definition \ref{P-minimal}}, states that there exists no Markovian DAG which can entail a proper superset of CI relations, and its unique variant further requires that all P-minimal DAGs belong to the same MEC as $\G^*$. We elaborate the importance of u-P-minimality in the next section. By \textbf{Theorem \ref{Chickering_seq}}, TSP guarantees that its output $\hat{\G}_\pi$ is P-minimal. When faithfulness holds, \textbf{Theorem \ref{razors}} ensures that $\hat{\G}_\pi \in \MEC(\G^*)$, and hence TSP is correct.

Notice that the identification of a Chickering sequence from $\G_\pi$ to a P-minimal $\G_\tau$ is essentially a DAG-based operation. In the next section, we introduce our permutation-based operation to converge to a P-minimal DAG, and propose a class of greedy permutation-based algorithms which employs weaker causal razors than TSP does.

In addition to TSP, \cite{solus2021consistency} introduced another greedy algorithm, namely \textit{Edge SP} (ESP), which is defined by weakly decreasing traversals over the \textit{DAG associahedron} (i.e., the \textit{permutohedron} contracted by $\I(\Prob)$). These technical terms are defined in the Appendix \textcolor{blue}{\textbf{\ref{app:ESP_GRaSP1}}}. ESP is shown to be assuming a weaker causal razor than TSP. In the next section, we will draw a logical discovery on how ESP is connected to our novel permutation-based operation.

\section{Methods}
\label{sec:method}
In this section, we introduce a class of permutation-based algorithms with a generic name \textit{Greedy Relaxations of Sparsest Permutation} (GRaSP). Three tiers of relaxation will be studied: GRaSP$_0$ is our basic algorithm, GRaSP$_1$ relaxes the search criterion of GRaSP$_0$ while GRaSP$_2$ further relaxes that of GRaSP$_1$. This hierarchy allows the identification of $\mt{MEC}(\G^*)$ under progressively weaker causal razors. In addition, we show that GRaSP$_0$ is logically equivalent to TSP, and GRaSP$_1$ to ESP. All proofs are left in Appendix \textcolor{blue}{\textbf{\ref{app:correct}}}-\textcolor{blue}{\textbf{\ref{app:tiers}}}. First, we introduce our characteristic permutation-based operation \textit{tuck} and how it operates under different types of directed edges.

\begin{definition}
\label{tuck}
(Tuck) Consider any graphoid $\Prob$ over $\mb{V}$, any $\pi \in \Pi(\mb{v})$, and any $j, k \in \mb{v}$ where $\pi[j] < \pi[k]$.
Rewrite $\pi$ as $\la \bs{\delta}_1, j, \bs{\delta}_2, k, \bs{\delta}_3 \ra$ where each $\bs{\delta}_i$ is a (possibly empty) sub-sequence of $\pi$.\footnote{To be precise, $\bs{\delta}_1 = \la \pi_i: 1 \leq i < \pi[j] \ra$, $\bs{\delta}_2 = \la \pi_i: \pi[j] < i < \pi[k]\ra$, and $\bs{\delta}_3 = \la \pi_i: \pi[k] < i \leq |\pi|\ra$.} Let $\bs{\gamma}$ and $\bs{\gamma}^c$ be the sub-sequences $\la i \in \bs{\delta}_2: i \in \mt{An}(k, \G_\pi)\ra$ and $\la i \in \bs{\delta}_2: i \notin \mt{An}(k, \G_\pi)\ra$ respectively. Define
\begin{align*}
    tuck(\pi, j, k) = 
    \begin{cases}
    \la \bs{\delta}_1, \bs{\gamma}, k, j, \bs{\gamma}^c, \bs{\delta}_3\ra & \text{ if } (j \to k) \in \E(\G_\pi)\\
    \pi & \text{ otherwise.}
    \end{cases}
\end{align*}
\end{definition}

\begin{definition}
\label{Et_edges}
Given a DAG $\G$, a directed edge $(j \to k) \in \E(\G)$ is said to be singular if there exists no unidirectional directed path from $j$ to $k$ in $\G$ except $j \to k$. Define
\begin{align*}
    \E^t(\G) =
    \begin{cases}
    \text{covered edges in } \E(\G) & \text{ if } t = 0\\
    \text{singular edges in } \E(\G) & \text{ if } t = 1\\
    \E(\G) & \text{ if } t = 2.
    \end{cases}
\end{align*}
\end{definition}

Readers can verify that $\E^0(\G) \subseteq \E^1(\G) \subseteq \E^2(\G)$ holds for any DAG $\G$. The introduction of singular edges is crucial to our logical discovery that every move ESP takes in the DAG associahedron (as defined in Appendix \textcolor{blue}{\textbf{\ref{app:ESP_GRaSP1}}}) corresponds to tucking a unique singular edge. Figure \ref{fig:tuck_illustration} provides an example on how \textit{tuck} works for each defined type of edges. As seen in the example, \textit{tuck} is an operation that aims to change a permutation \textit{minimally} to obtain a differently induced DAG, while a broader class of directed edges generally leads to more possible re-orderings of the vertices.

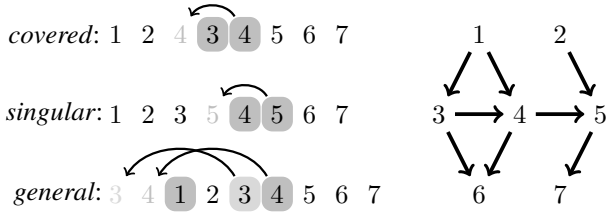
\begin{figure}[ht!]
    \centering
    \begin{tikzpicture}[scale=0.85, roundnode/.style={circle, draw=black!60, very thick, minimum size=5mm}]
    \node (X1) at (5.625, 2.5) {$1$};
    \node (X2) at (6.875, 2.5) {$2$};
    \node (X3) at (5, 1.25) {$3$};
    \node (X4) at (6.25, 1.25) {$4$};
    \node (X5) at (7.5, 1.25) {$5$};
    \node (X6) at (5.625, 0) {$6$};
    \node (X7) at (6.875, 0) {$7$};
    \path [->,line width=0.5mm] (X1) edge (X3);
    \path [->,line width=0.5mm] (X1) edge (X4);
    \path [->,line width=0.5mm] (X2) edge (X5);
    \path [->,line width=0.5mm] (X3) edge (X4);
    \path [->,line width=0.5mm] (X4) edge (X5);
    \path [->,line width=0.5mm] (X3) edge (X6);
    \path [->,line width=0.5mm] (X4) edge (X6);
    \path [->,line width=0.5mm] (X5) edge (X7);
    \node (c) at (-0.95, 2.5) {\textit{covered}:};
    \node (s) at (-0.95, 1.25) {\textit{singular}:};
    \node (g) at (-0.9, 0) {\textit{general}:};
    \node (c1) at (0.0, 2.5) {$1$};
    \node (c2) at (0.5, 2.5) {$2$};
    \node [text=lightgray](c3) at (1.0, 2.5) {$4$};
    \node [fill=lightgray, rounded corners](c4) at (1.5, 2.5) {$3$};
    \node [fill=lightgray, rounded corners](c5) at (2.0, 2.5) {$4$};
    \node (c6) at (2.5, 2.5) {$5$};
    \node (c7) at (3.0, 2.5) {$6$};
    \node (c8) at (3.5, 2.5) {$7$};
    \node (s1) at (0.0, 1.25) {$1$};
    \node (s2) at (0.5, 1.25) {$2$};
    \node (s3) at (1.0, 1.25) {$3$};
    \node [text=lightgray](s4) at (1.5, 1.25) {$5$};
    \node [fill=lightgray, rounded corners](s5) at (2.0, 1.25) {$4$};
    \node [fill=lightgray, rounded corners](s6) at (2.5, 1.25) {$5$};
    \node (s7) at (3.0, 1.25) {$6$};
    \node (s8) at (3.5, 1.25) {$7$};
    \node [text=gray!30](ns1) at (0.0, 0) {$3$};
    \node [text=lightgray](ns2) at (0.5, 0) {$4$};
    \node [fill=lightgray, rounded corners](ns3) at (1.0, 0) {$1$};
    \node (ns4) at (1.5, 0) {$2$};
    \node [fill=gray!30, rounded corners](ns5) at (2.0, 0) {$3$};
    \node [fill=lightgray, rounded corners](ns6) at (2.5, 0) {$4$};
    \node (ns7) at (3.0, 0) {$5$};
    \node (ns8) at (3.5, 0) {$6$};
    \node (ns9) at (4.0, 0) {$7$};
    \path [->,line width=0.3mm,bend right=60] (c5) edge (c3);
    \path [->,line width=0.3mm,bend right=60] (s6) edge (s4);
    \path [->,line width=0.3mm,bend right=60] (ns5) edge (ns1);
    \path [->,line width=0.3mm,bend right=60] (ns6) edge (ns2);
    \end{tikzpicture}
    \caption{Consider $\pi = \la 1, 2, 3, 4, 5, 6, 7\ra$ and its induced $\G_\pi$ shown on the right. Each of the three orderings on the left illustrates how a directed edge between two darkly shaded vertices is tucked to obtain a new permutation. For example, consider $1 \to 4$ which is \textit{not} singular due to the unidirectional path $1 \to 3 \to 4$. Performing $\textit{tuck}(\pi, 1, 4)$ requires the identification of the intermediate vertices between $1$ and $4$ in $\pi$ which are ancestors of 4 in $\G_\pi$ (i.e., the lightly shaded $3$). Then, while the positions of other vertices remain intact, $3$ and $4$ are moved to the front of $1$.}
    \label{fig:tuck_illustration}
\end{figure}

After clarifying how \textit{tuck} works, we can define a sequence of \textit{tuck} operations, particularly when applied to covered edges, and how 
a sequence of \textit{covered tucks} (\textit{ct}) is connected to a Chickering sequence.

\begin{definition}
\label{ct-seq}
(ct-sequence) Given a graphoid $\Prob$ over $\mb{V}$, for any $\pi, \tau \in \Pi(\mb{v})$, $\tau$ is said to be a ct-mutation of $\pi$ if there exist $j, k \in \mb{v}$ s.t. $(j \to k) \in \E(\G_\pi)$ is covered and $\tau = \textit{tuck}(\pi, j, k)$. Also, $\la \pi^1,...,\pi^m\ra$ is said to be a ct-sequence if $\pi^{i+1}$ is a ct-mutation of $\pi^i$ for each $1 \leq i < m$, and $(\G_{\pi^i}, \G_{\pi^l})$ are pairwise distinct for any $1 \leq i < l \leq m$. 
\end{definition}

\begin{lemma}
\label{ct-better}
[Appendix \textcolor{blue}{\textbf{\ref{app:correct}}}] Given a graphoid $\Prob$, for any $\pi \in \Pi(\mb{v})$ and any Chickering sequence from $\G_\pi$ to some $\mc{H} \in \SGS(\Prob)$ considered by TSP, there exists a ct-sequence $\la \pi,...,\tau\ra$ s.t. $\G_\tau = \mc{H}$.
\end{lemma}

Similar to the DAG-based DFS over Chickering sequences employed by TSP, the lemma above motivates our permutation-based DFS over ct-sequences as shown in \textbf{Algorithm \ref{alg:dfs}}. 

\begin{algorithm}[!ht]
\DontPrintSemicolon
\caption{\textsc{DFS: }$\textit{dfs}(\Prob, \pi, d, d_\textit{cur}, t)$}
\label{alg:dfs}
\KwIn{(a) $\Prob$: a graphoid over $\mb{V}$; (b) $\pi \in \Pi(\mb{v}$); (c) $d$: depth bound; (d) $d_\textit{cur}$: recorder of the recursive call; (e) $t$: type of directed edges}
\KwOut{$\tau \in \Pi(\mb{v}$) where $\textit{score}(\tau) \geq \textit{score}(\pi)$}
\ForEach{$(j \to k) \in \E^t(\G_\pi)$}{
    $\tau \ot \textit{tuck}(\pi, j, k)$ \;
    \If{$\textit{score}(\tau) = \textit{score}(\pi)$ and $d_\textit{cur} < d$}{
        $\tau \ot \textit{dfs}(\Prob, \tau, d, d_\textit{cur} + 1, t)$ \;
    }
    \If{$\textit{score}(\tau) > \textit{score}(\pi)$}{
        return $\tau$
    }
}
return $\pi$    
\end{algorithm}

\begin{algorithm}[!ht]
\DontPrintSemicolon
\caption{\textsc{GR}a\textsc{SP}$_t$: $\textit{grasp}(\Prob, \pi, d, t)$}
\label{alg:grasp}
\KwIn{(a) $\Prob$: a graphoid over $\mb{V}$; (b) $\pi \in \Pi(\mb{v})$; (c) $d$: depth bound; (d) $t$: tier of GRaSP}
\KwOut{$\tau \in \Pi(\mb{v}$) where $\textit{score}(\tau) \geq \textit{score}(\pi)$}
\If{$t \neq 0$}{
$\pi = \textit{grasp}(\Prob, \pi, d, t-1)$}
$\tau \ot \pi$\;
\Do{$\textit{score}(\tau) > \textit{score}(\pi)$}{
    $\pi \ot \tau$\;
    $\tau \ot \textit{dfs}(\Prob, \pi, d, 1, t)$\;
}
return $\tau$
\end{algorithm}

    First, we use \textit{negative edge count} as the scoring function in our oracle version of the algorithm such that $\textit{score}(\pi) = - |\E(\G_\pi)|$ where $\G_\pi$ is induced from $\pi$ and $\Prob$. $d$ bounds the search depth of DFS. We assume that $d = |\mb{v}|!$ for now and call the corresponding algorithm \textit{unbounded}. We will examine some small number $d$ in light of finite samples in Section \textcolor{blue}{\textbf{\ref{Linear_Gauss_Sim}}}. Also, we assume that no induced DAG can be revisited in the DFS procedure in order to avoid any infinite loop between DAGs. 

Next, $\E^t(\G_\pi)$, as defined in \textbf{Definition \ref{Et_edges}}, is the crucial function distinguishing our three tiers of GRaSP in \textbf{Algorithm \ref{alg:grasp}}. Consider $t = 0$ in particular. Given an arbitrary initial permutation $\pi$, \textbf{Algorithm \ref{alg:dfs}} performs a greedy procedure to identify a ct-sequence from $\pi$. Figure \ref{fig:ct_seq} shows a simple example. Then \textbf{Algorithm \ref{alg:grasp}} iterates the DFS in \textbf{Algorithm \ref{alg:dfs}} until no sparser permutation can be found. The theorem below ensures that $\hat{\tau}$ returned by \textbf{Algorithm \ref{alg:grasp}} induces $\G_{\hat{\tau}} \in \Pm(\Prob)$.

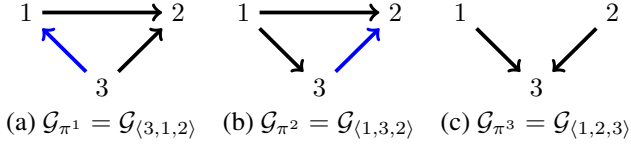
\begin{figure}
\begin{center}
\subfloat{
\begin{tikzpicture}[roundnode/.style={circle, draw=black!60, very thick, minimum size=5mm}]
\node(X1) at (0,0) {$1$};
\node(X2) at (2,0) {$2$};
\node(X3) at (1,-1) {$3$};
\node(G) at (1, -1.5) {(a) $\mc{G}_{\pi^1} = \mc{G}_{\la 3, 1, 2\ra}$};
\path [->,line width=0.5mm] (X1) edge (X2);
\path [->,line width=0.5mm, blue] (X3) edge (X1);
\path [->,line width=0.5mm] (X3) edge (X2);
\end{tikzpicture}}
\subfloat{
\begin{tikzpicture}[roundnode/.style={circle, draw=black!60, very thick, minimum size=5mm}]
\node(X1) at (0,0) {$1$};
\node(X2) at (2,0) {$2$};
\node(X3) at (1,-1) {$3$};
\node(G) at (1, -1.5) {(b) $\mc{G}_{\pi^2} = \mc{G}_{\la 1, 3, 2\ra}$};
\path [->,line width=0.5mm] (X1) edge (X2);
\path [->,line width=0.5mm] (X1) edge (X3);
\path [->,line width=0.5mm, blue] (X3) edge (X2);
\end{tikzpicture}}
\subfloat{
\begin{tikzpicture}[roundnode/.style={circle, draw=black!60, very thick, minimum size=5mm}]
\node(X1) at (0,0) {$1$};
\node(X2) at (2,0) {$2$};
\node(X3) at (1,-1) {$3$};
\node(G) at (1, -1.5) {(c) $\mc{G}_{\pi^3} = \mc{G}_{\la 1, 2, 3\ra}$};
\path [->,line width=0.5mm] (X1) edge (X3);
\path [->,line width=0.5mm] (X2) edge (X3);
\end{tikzpicture}}
\end{center}
\caption{Example of a ct-sequence $\la \pi^1, \pi^2, \pi^3\ra$ where $\I(\Prob) = \{\la X_1, X_2\,|\,\varnothing\ra\}$. The blue (covered) edges indicate how a subsequent permutation is obtained by \textit{tuck}. For example, $3 \to 1$ in (a) specifies that $\pi^2$ is obtained from $\textit{tuck}(\pi^1, 3, 1)$. Also, \textbf{Algorithm \ref{alg:dfs}} returns $\pi^3 = \la 1,2,3\ra$ since the DAG in (c) is sparser than those in (a) and (b).}
\label{fig:ct_seq}
\end{figure}

\begin{theorem}
\label{ct-theorem}
[Appendix \textcolor{blue}{\textbf{\ref{app:correct}}}] Given a graphoid $\Prob$ over $\mb{V}$ and any $\pi \in \Pi(\mb{v})$, if $\G_\pi \notin \Pm(\Prob)$, then there exists a ct-sequence $\mf{T} = \la \pi,..., \tau\ra$ s.t. $\G_\tau \in \Pm(\Prob)$.
\end{theorem}

By \textbf{Theorem \ref{ct-theorem}}, the correctness of unbounded GRaSP$_0$ under faithfulness follows immediately from \textbf{Theorem \ref{razors}}. As shown by \cite{forster2020frugal}, $\CFC(\Prob) = \Fr(\Prob)$ holds under faithfulness. Since \textbf{Algorithm \ref{alg:grasp}} requires that the permutation returned by a higher tier of GRaSP cannot be denser than that returned by a lower tier, the correctness of unbounded GRaSP$_1$ and unbounded GRaSP$_2$ under faithfulness immediately follows. The sample version of GRaSP can be obtained by substituting the graphoid $\Prob$ with an \textit{i.i.d.} observational dataset $\mc{D}$, and $\textit{score}(\pi)$ with the BIC score of $\G_\pi$ from $\mc{D}$ (defined in Appendix \textcolor{blue}{\textbf{\ref{app:gs}}}). Pointwise consistency under faithfulness directly follows from the \textit{local consistency} of BIC.\footnote{See \citep{10.1214/aos/1176350709} and \citep{chickering2002optimal} for the (local) consistency of BIC.}

\begin{corollary}
\label{correct_consistent}
Unbounded GRaSP$_0$, GRaSP$_1$, and GRaSP$_2$ are correct and pointwise consistent under faithfulness.
\end{corollary}

Next, we want to highlight two logical discoveries with respect to the discussion of TSP and GRaSP$_0$.  

\begin{theorem}
\label{TSP=GRaSP0}
[Appendix \textcolor{blue}{\textbf{\ref{app:correct}}}] Given a graphoid $\Prob$ and an initial permutation, the DAG returned by TSP is the same as the DAG induced by the output of unbounded GRaSP$_0$. 
\end{theorem}

The theorem above suggests that TSP and GRaSP$_0$ are \textit{logically equivalent}. Additionally,  contrary to what \cite{solus2021consistency} argued, faithfulness is a necessary condition for TSP. 

\begin{theorem}
\label{TSP_necessary}
[Appendix \textcolor{blue}{\textbf{\ref{app:correct}}}] Given a graphoid $\Prob$, faithfulness is necessary for the correctness of TSP.
\end{theorem}

This theorem is entailed by a novel logical result that $\CFC(\Prob) = \uPm(\Prob)$ as proven in Appendix \textcolor{blue}{\textbf{\ref{app:correct}}}. Thus, the two theorems together prompt the usage of GRaSP with a higher tier. Extending $\E^0(\cdot)$ to $\E^1(\cdot)$ and $\E^2(\cdot)$ licenses a higher tier of GRaSP to attain a strictly sparser permutation under unfaithfulness. Examples of this sort will be studied in Section \textcolor{blue}{\textbf{\ref{u_fru_unfaithful}}} and Appendix \textcolor{blue}{\textbf{\ref{app:tiers}}}. 

\begin{corollary}
\label{coro_GRaSP_hierarchy}
Given a graphoid $\Prob$, unbounded GRaSP$_2$ is correct under a strictly weaker causal razor than unbounded GRaSP$_1$, which is correct under a strictly weaker causal razor than unbounded GRaSP$_0$.
\end{corollary}

Further, in Appendix \textcolor{blue}{\textbf{\ref{app:ESP_GRaSP1}}}, we show the logical equivalence between unbounded GRaSP$_1$ and ESP. As a consequence, unbounded GRaSP$_2$ is a relaxation beyond the two causal razors discussed in \citep{solus2021consistency}. That said, we are aware of cases where unbounded GRaSP$_2$ is incorrect under u-frugality. Such a counterexample will be studied in Section \textcolor{blue}{\textbf{\ref{u_fru_unfaithful}}} and Appendix \textcolor{blue}{\textbf{\ref{app:tiers}}}.

We conclude this section by discussing how to use the DAG-inducing method in \citep{verma1988causal} based on BIC scores. This facilitates our simulations done in Section \textcolor{blue}{\textbf{\ref{Linear_Gauss_Sim}}}. Given a semigraphoid $\Prob$ over $\mb{V}$, each $\pi \in \Pi(\mb{v})$ induces a DAG $\G_\pi$ satisfying the following condition:
\begin{align}\tag{VP}
    X_j \in \mb{M} \Leftrightarrow (j \to k) \in \E(\G_\pi) 
\end{align}
where $\mb{M}$ is a \textit{Markov boundary} of $X_k$ relative to $\mb{X}_{\Pre(k, \pi)}$ (defined in Appendix \textcolor{blue}{\textbf{\ref{app:DAG_induce}}}). \textbf{Lemma \ref{VP=RU}} highlights that the DAGs induced by (VP) and (RU) are equivalent when $\Prob$ is a graphoid. But (VP) is preferred since we can estimate the \textit{unique} Markov boundary by the \textit{Grow-Shrink} (GS) algorithm from \citep{margaritis1999bayesian} using BIC scores and avoid hypothesis testing needed in (RU). We leave the discussion of the GS algorithm in Appendix \textcolor{blue}{\textbf{\ref{app:gs}}}. In Section \textcolor{blue}{\textbf{\ref{Linear_Gauss_Sim}}}, we are going to evaluate the performance of GRaSP through (VP) and GS in light of finite samples.



\section{Simulations}
\label{sec:sims}
In this section, we review empirical results of unfaithful u-frugal models with respect to DAGs and algorithmic performance on Gaussian distributed data generated under a variety of situations. References to the code and instantiated models with replicability instructions are included on a GitHub site for the project\footnote{\url{https://github.com/cmu-phil/grasp}.}. Also referenced will be a running version of GRaSP in the Tetrad project (\cite{ramsey2018tetrad}) as well as tabular data for all simulations. A scalable Python translation of GRaSP$_2$ using (VP) with a linear, Gaussian BIC score 
is included in the causal-learn Python package.\footnote{\url{https://github.com/cmu-phil/causal-learn}.}


\subsection{U-Frugal Faithfulness Violations}
\label{u_fru_unfaithful}

In what follows, we consider three sets of u-frugal models that violate faithfulness. The sets of models correspond to: regular Gaussian distributions over four variables \citep{vsimecek2006gaussian}, discrete distributions over four variables satisfying the intersection graphoid axiom and the Spohn condition (this includes all positive discrete distributions) \citep{vsimecek2006short} (see Appendix \textcolor{blue}{\textbf{\ref{app:graphoid}}}), and unfaithful DAGs (uDAGs) over five variables where a path cancellation induces a marginal independence between a pair of variables (see Appendix \textcolor{blue}{\textbf{\ref{app:unit_tests}}})\footnote{The first two sets of models can be found at \url{http://5r.matfyz.cz/skola/models}.}. In Table \ref{tab:unit_test}, these sets are denoted Gaussian, Discrete, and uDAGs, respectively.

We evaluate the capabilities of GRaSP$_0$, GRaSP$_1$, and GRaSP$_2$ to recover u-frugal DAGs using an independence oracle on models from each set. We say that a GRaSP variant recovers the u-frugal model if it can do so from every permutation; if the algorithm can reach the u-frugal model from every permutation, then the correctness of the variant will be independent of the DFS implementation.

\begin{table}[ht!]
    \centering
    \begin{tabular}{c|c|c|c|c}
    	& GRaSP$_0$ & GRaSP$_1$ & GRaSP$_2$ & Total \\
    	\hline
        Gaussian & 0 & 7 & 10 & 10 \\
        Discrete & 0 & 79 & 84 & 84 \\
        uDAGs & 0 & 19 & 49 & 61
    \end{tabular}
    \caption{The number of u-frugal models recovered by GRaSP$_0$, GRaSP$_1$, and GRaSP$_2$ from three sets of u-frugal models that violate faithfulness. A model is considered to be recovered if it is recovered from every permutation.}
    \label{tab:unit_test}
\end{table}

Table \ref{tab:unit_test} provides a computational proof that there are GRaSP$_1$ models not found by GRaSP$_0$, and GRaSP$_2$ models not found by GRaSP$_1$. These results support the claims in \textbf{Corollary \ref{coro_GRaSP_hierarchy}}.

\subsection{Linear Gaussian Simulations}
\label{Linear_Gauss_Sim}
We studied GRaSP's performance in the linear Gaussian case by varying simulations parameters around a configuration with 60 variables, an average degree of 6, and a sample size of 1,000 against two standard algorithms: fGES \citep{chickering2002optimal, ramsey2017million} and PC \citep{spirtes2000causation}. In Figure \ref{fig:mVar}, we vary the number of measured variables from 20 to 100 with values 20, 30, 40, 50, 60, 70, 80, 90, and 100. In Figure \ref{fig:avgDeg}, we vary the average degree from 2 to 10 with values 2, 3, 4, 5, 6, 7, 8, 9, and 10. For Figure \ref{fig:sampSize}, we vary the sample size from 200 to 100,000, with values 200, 500, 1,000, 2,000, 5,000, 10,000, 20,000, 50,000, and 100,000. In all cases, we draw coefficient values uniformly from $U(-1, 1)$ and incorporate independent additive exogenous noise distributions set to $N(0, 1)$.
All statistics are averaged over 20 independent runs. Finally, in Figure \ref{fig:secs}, we give the running times for our Java implementation of the algorithms. All of the algorithms except PC used BIC with a parameter penalty multiplier of 2 as a score; PC used partial correlation with a significance threshold of 0.001 as a conditional independence test. For the GRaSP variants, we allow tucks of covered edges up to depth 3, and tucks of non-covered edges at depth 1 when applicable\footnote{In the Java implementation of the algorithm, we include parameters for uncovered depth and non-singular depth to provide the user with more control over this heuristic.}. In all cases, we follow the procedure set out in the text of running lower tiers of GRaSP before running higher tiers of GRaSP to guarantee consistent improvement of statistics. 

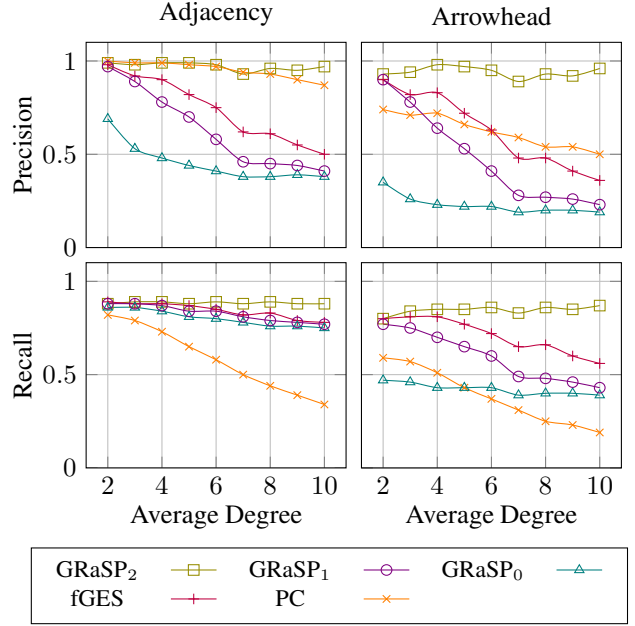
\begin{figure}[ht!]
    \centering
    \begin{tikzpicture}
        \begin{groupplot}[
            group style={
                group size=2 by 2,
                group name=avgDeg,
                x descriptions at=edge bottom,
                y descriptions at=edge left,
                horizontal sep=2mm,
                vertical sep=2mm
            },
            xlabel=Average Degree,
            ymin=0, ymax=1.1,
            grid=both,
            xlabel style={yshift=1mm},
            ylabel style={yshift=-1mm},
            width=5cm
        ]
        
            \nextgroupplot[ylabel=Precision]
            \addplot[mark=square, smooth, color=olive] table[x=avgDegree, y=GRaSP_AP]{Figures/avgDegree.txt}; \label{avgDeg_grasp}
            \addplot[mark=x, smooth, color=orange] table[x=avgDegree, y=PC_AP]{Figures/avgDegree.txt}; \label{avgDeg_pc}
            \addplot[mark=+, smooth, color=purple] table[x=avgDegree, y=fGES_AP]{Figures/avgDegree.txt}; \label{avgDeg_fges}
            \addplot[mark=triangle, smooth, color=teal] table[x=avgDegree, y=TSP_AP]{Figures/avgDegree.txt}; \label{avgDeg_tsp}
            \addplot[mark=o, smooth, color=violet] table[x=avgDegree, y=ESP_AP]{Figures/avgDegree.txt}; \label{avgDeg_esp}
            
            \nextgroupplot
            \addplot[mark=square, smooth, color=olive] table[x=avgDegree, y=GRaSP_AHP]{Figures/avgDegree.txt};
            \addplot[mark=x, smooth, color=orange] table[x=avgDegree, y=PC_AHP]{Figures/avgDegree.txt};
            \addplot[mark=+, smooth, color=purple] table[x=avgDegree, y=fGES_AHP]{Figures/avgDegree.txt};
            \addplot[mark=triangle, smooth, color=teal] table[x=avgDegree, y=TSP_AHP]{Figures/avgDegree.txt};
            \addplot[mark=o, smooth, color=violet] table[x=avgDegree, y=ESP_AHP]{Figures/avgDegree.txt};
            
            \nextgroupplot[ylabel=Recall]
            \addplot[mark=square, smooth, color=olive] table[x=avgDegree, y=GRaSP_AR]{Figures/avgDegree.txt};
            \addplot[mark=x, smooth, color=orange] table[x=avgDegree, y=PC_AR]{Figures/avgDegree.txt};
            \addplot[mark=+, smooth, color=purple] table[x=avgDegree, y=fGES_AR]{Figures/avgDegree.txt};
            \addplot[mark=triangle, smooth, color=teal] table[x=avgDegree, y=TSP_AR]{Figures/avgDegree.txt};
            \addplot[mark=o, smooth, color=violet] table[x=avgDegree, y=ESP_AR]{Figures/avgDegree.txt};
            
            \nextgroupplot
            \addplot[mark=square, smooth, color=olive] table[x=avgDegree, y=GRaSP_AHR]{Figures/avgDegree.txt};
            \addplot[mark=x, smooth, color=orange] table[x=avgDegree, y=PC_AHR]{Figures/avgDegree.txt};
            \addplot[mark=+, smooth, color=purple] table[x=avgDegree, y=fGES_AHR]{Figures/avgDegree.txt};
            \addplot[mark=triangle, smooth, color=teal] table[x=avgDegree, y=TSP_AHR]{Figures/avgDegree.txt};
            \addplot[mark=o, smooth, color=violet] table[x=avgDegree, y=ESP_AHR]{Figures/avgDegree.txt};
            
        \end{groupplot}
        \node[above=1mm of avgDeg c1r1] {Adjacency};
        \node[above=1mm of avgDeg c2r1] {Arrowhead};
        \node[fill=white, draw=black] at (3.2, -4.5) {
            \small
            \begin{tabular}{c c c c c c}
                GRaSP$_2$ & \ref{avgDeg_grasp} & GRaSP$_1$ & \ref{avgDeg_esp} & GRaSP$_0$ & \ref{avgDeg_tsp} \\
                fGES & \ref{avgDeg_fges} & PC & \ref{avgDeg_pc}
            \end{tabular}
        };
    \end{tikzpicture}
    \caption{Average degree varied, measured variables fixed to 60, sample size fixed to 1,000.}
    \label{fig:avgDeg}
\end{figure}

In these figures, precision = $\TP / (\TP + \FP)$ and recall = $\TP / (\TP + \FN)$, where $\TP$ is the number of true positives, $\FP$ is the number of false positives, and $\FN$ is the number of false negatives. We give precision and recall statistics for adjacencies and arrowheads separately. For adjacencies, true (false) adjacencies are pairs of vertices that are (not) adjacent in the generative graphical model, and positive (negative) adjacencies are pairs of vertices that are (not) adjacent in the estimated graphical model for each algorithm, respectively. For arrowhead statistics, a true arrowhead is a directed edge in the CPDAG\footnote{A CPDAG (a.k.a. ``pattern'') is a graphical representation of the Markov equivalence class for a DAG. See \citep{spirtes2000causation} for details.} of the generative graphical model and a positive arrowhead is a directed edge in the CPDAG of the estimated DAG, with negative and false arrowheads indicating the absence of these directed edges in their respective CPDAGs.

Figure \ref{fig:avgDeg} shows that algorithmic performance is strongly dependent on the average degree. While the compared algorithms generally perform well on sparse models, their performance drops off as the density increases. The exception is GRaSP$_2$, which dominates this group of algorithms, with a strong performance for both adjacencies and arrowheads as average degree is increased.

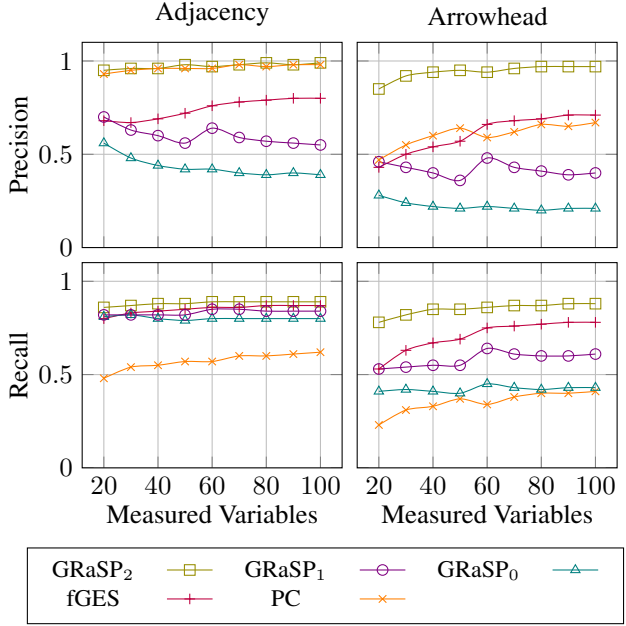
\begin{figure}[ht!]
    \centering
    \begin{tikzpicture}
        \begin{groupplot}[
            group style={
                group size=2 by 2,
                group name=mVar,
                x descriptions at=edge bottom,
                y descriptions at=edge left,
                horizontal sep=2mm,
                vertical sep=2mm
            },
            xlabel=Measured Variables,
            ymin=0, ymax=1.1,
            grid=both,
            xlabel style={yshift=1mm},
            ylabel style={yshift=-1mm},
            width=5cm
        ]
        
            \nextgroupplot[ylabel=Precision]
            \addplot[mark=square, smooth, color=olive] table[x=numMeasures, y=GRaSP_AP]{Figures/numMeasures.txt}; \label{mVar_grasp}
            \addplot[mark=x, smooth, color=orange] table[x=numMeasures, y=PC_AP]{Figures/numMeasures.txt}; \label{mVar_pc}
            \addplot[mark=+, smooth, color=purple] table[x=numMeasures, y=fGES_AP]{Figures/numMeasures.txt}; \label{mVar_fges}
            \addplot[mark=triangle, smooth, color=teal] table[x=numMeasures, y=TSP_AP]{Figures/numMeasures.txt}; \label{mVar_tsp}
            \addplot[mark=o, smooth, color=violet] table[x=numMeasures, y=ESP_AP]{Figures/numMeasures.txt}; \label{mVar_esp}
            
            \nextgroupplot
            \addplot[mark=square, smooth, color=olive] table[x=numMeasures, y=GRaSP_AHP]{Figures/numMeasures.txt};
            \addplot[mark=x, smooth, color=orange] table[x=numMeasures, y=PC_AHP]{Figures/numMeasures.txt};
            \addplot[mark=+, smooth, color=purple] table[x=numMeasures, y=fGES_AHP]{Figures/numMeasures.txt};
            \addplot[mark=triangle, smooth, color=teal] table[x=numMeasures, y=TSP_AHP]{Figures/numMeasures.txt};
            \addplot[mark=o, smooth, color=violet] table[x=numMeasures, y=ESP_AHP]{Figures/numMeasures.txt};
            
            \nextgroupplot[ylabel=Recall]
            \addplot[mark=square, smooth, color=olive] table[x=numMeasures, y=GRaSP_AR]{Figures/numMeasures.txt};
            \addplot[mark=x, smooth, color=orange] table[x=numMeasures, y=PC_AR]{Figures/numMeasures.txt};
            \addplot[mark=+, smooth, color=purple] table[x=numMeasures, y=fGES_AR]{Figures/numMeasures.txt};
            \addplot[mark=triangle, smooth, color=teal] table[x=numMeasures, y=TSP_AR]{Figures/numMeasures.txt};
            \addplot[mark=o, smooth, color=violet] table[x=numMeasures, y=ESP_AR]{Figures/numMeasures.txt};
            
            \nextgroupplot
            \addplot[mark=square, smooth, color=olive] table[x=numMeasures, y=GRaSP_AHR]{Figures/numMeasures.txt};
            \addplot[mark=x, smooth, color=orange] table[x=numMeasures, y=PC_AHR]{Figures/numMeasures.txt};
            \addplot[mark=+, smooth, color=purple] table[x=numMeasures, y=fGES_AHR]{Figures/numMeasures.txt};
            \addplot[mark=triangle, smooth, color=teal] table[x=numMeasures, y=TSP_AHR]{Figures/numMeasures.txt};
            \addplot[mark=o, smooth, color=violet] table[x=numMeasures, y=ESP_AHR]{Figures/numMeasures.txt};
            
        \end{groupplot}
        \node[above=1mm of mVar c1r1] {Adjacency};
        \node[above=1mm of mVar c2r1] {Arrowhead};
        \node[fill=white, draw=black] at (3.2, -4.5) {
            \small
            \begin{tabular}{c c c c c c}
                GRaSP$_2$ & \ref{mVar_grasp} & GRaSP$_1$ & \ref{mVar_esp} & GRaSP$_0$ & \ref{mVar_tsp} \\
                fGES & \ref{mVar_fges} & PC & \ref{mVar_pc}
            \end{tabular}
        };
    \end{tikzpicture}
    \caption{Measured variables varied, average degree fixed to 6, sample size fixed to 1,000.}
    \label{fig:mVar}
\end{figure}

Figure \ref{fig:mVar} shows the result of varying the number of measured variables. Notably, increasing the number of measured variables while holding the average degree constant decreases graph density. We see upward trends for some arrowhead statistics corresponding to this decrease in density. Again, GRaSP$_2$ dominates this group of algorithms, with strong precision and recall for both adjacencies and arrowheads.

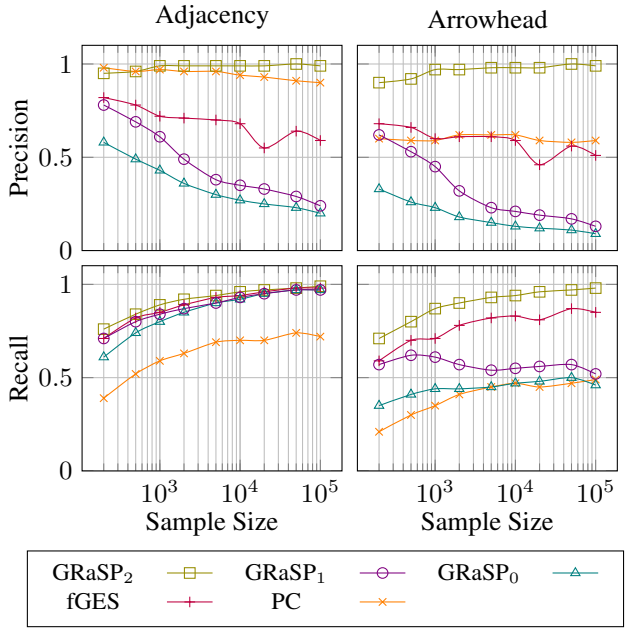
\begin{figure}[ht!]
    \centering
    \begin{tikzpicture}
        \begin{groupplot}[
            group style={
                group size=2 by 2,
                group name=sampSize,
                x descriptions at=edge bottom,
                y descriptions at=edge left,
                horizontal sep=2mm,
                vertical sep=2mm
            },
            xlabel=Sample Size,
            ymin=0, ymax=1.1,
            grid=both,
            xmode=log, 
            xlabel style={yshift=1mm},
            ylabel style={yshift=-1mm},
            width=5cm
        ]
        
            \nextgroupplot[ylabel=Precision]
            \addplot[mark=square, smooth, color=olive] table[x=sampleSize, y=GRaSP_AP]{Figures/sampleSize.txt}; \label{sampSize_grasp}
            \addplot[mark=x, smooth, color=orange] table[x=sampleSize, y=PC_AP]{Figures/sampleSize.txt}; \label{sampSize_pc}
            \addplot[mark=+, smooth, color=purple] table[x=sampleSize, y=fGES_AP]{Figures/sampleSize.txt}; \label{sampSize_fges}
            \addplot[mark=triangle, smooth, color=teal] table[x=sampleSize, y=TSP_AP]{Figures/sampleSize.txt}; \label{sampSize_tsp}
            \addplot[mark=o, smooth, color=violet] table[x=sampleSize, y=ESP_AP]{Figures/sampleSize.txt}; \label{sampSize_esp}
            
            \nextgroupplot
            \addplot[mark=square, smooth, color=olive] table[x=sampleSize, y=GRaSP_AHP]{Figures/sampleSize.txt};
            \addplot[mark=x, smooth, color=orange] table[x=sampleSize, y=PC_AHP]{Figures/sampleSize.txt};
            \addplot[mark=+, smooth, color=purple] table[x=sampleSize, y=fGES_AHP]{Figures/sampleSize.txt};
            \addplot[mark=triangle, smooth, color=teal] table[x=sampleSize, y=TSP_AHP]{Figures/sampleSize.txt};
            \addplot[mark=o, smooth, color=violet] table[x=sampleSize, y=ESP_AHP]{Figures/sampleSize.txt};
            
            \nextgroupplot[ylabel=Recall]
            \addplot[mark=square, smooth, color=olive] table[x=sampleSize, y=GRaSP_AR]{Figures/sampleSize.txt};
            \addplot[mark=x, smooth, color=orange] table[x=sampleSize, y=PC_AR]{Figures/sampleSize.txt};
            \addplot[mark=+, smooth, color=purple] table[x=sampleSize, y=fGES_AR]{Figures/sampleSize.txt};
            \addplot[mark=triangle, smooth, color=teal] table[x=sampleSize, y=TSP_AR]{Figures/sampleSize.txt};
            \addplot[mark=o, smooth, color=violet] table[x=sampleSize, y=ESP_AR]{Figures/sampleSize.txt};
            
            \nextgroupplot
            \addplot[mark=square, smooth, color=olive] table[x=sampleSize, y=GRaSP_AHR]{Figures/sampleSize.txt};
            \addplot[mark=x, smooth, color=orange] table[x=sampleSize, y=PC_AHR]{Figures/sampleSize.txt};
            \addplot[mark=+, smooth, color=purple] table[x=sampleSize, y=fGES_AHR]{Figures/sampleSize.txt};
            \addplot[mark=triangle, smooth, color=teal] table[x=sampleSize, y=TSP_AHR]{Figures/sampleSize.txt};
            \addplot[mark=o, smooth, color=violet] table[x=sampleSize, y=ESP_AHR]{Figures/sampleSize.txt};
            
        \end{groupplot}
        \node[above=1mm of sampSize c1r1] {Adjacency};
        \node[above=1mm of sampSize c2r1] {Arrowhead};
        \node[fill=white, draw=black] at (3.2, -4.5) {
            \small
            \begin{tabular}{c c c c c c}
                GRaSP$_2$ & \ref{sampSize_grasp} & GRaSP$_1$ & \ref{sampSize_esp} & GRaSP$_0$ & \ref{sampSize_tsp} \\
                fGES & \ref{sampSize_fges} & PC & \ref{sampSize_pc}
            \end{tabular}
        };
    \end{tikzpicture}
    \caption{Sample size varied, measured variables fixed to 60, average degree fixed to 6.}
    \label{fig:sampSize}
\end{figure}

All the algorithms compared in these simulations claim pointwise consistency, however, only GRaSP$_2$ seems to back up these claims in Figure \ref{fig:sampSize}. This might suggest that GRaSP$_2$ is better equipped to handle almost-violations of faithfulness in linear Gaussian models. As with previous figures, GRaSP$_2$ dominates this group of algorithms for precision and recall for both adjacencies and arrowheads for all sample sizes studied.

\begin{figure}[ht!]
    \centering
    \begin{tikzpicture}
        \begin{groupplot}[
            group style={
                group size=2 by 1,
                group name=secs,
                x descriptions at=edge bottom,
                y descriptions at=edge left,
                horizontal sep=2mm,
                vertical sep=2mm
            },
            ylabel=Seconds,
            ymin=0.01, ymax=300,
            grid=both,
            ymode=log, 
            xlabel style={yshift=1mm},
            ylabel style={xshift=3mm, yshift=-3mm},
            width=5cm      
        ]
        
            \nextgroupplot[xlabel=Average Degree]
            \addplot[mark=square, smooth, color=olive] table[x=avgDegree, y=GRaSP_Secs]{Figures/avgDegree.txt}; \label{secs_grasp}
            \addplot[mark=x, smooth, color=orange] table[x=avgDegree, y=PC_Secs]{Figures/avgDegree.txt}; \label{secs_pc}
            \addplot[mark=+, smooth, color=purple] table[x=avgDegree, y=fGES_Secs]{Figures/avgDegree.txt}; \label{secs_fges}
            \addplot[mark=triangle, smooth, color=teal] table[x=avgDegree, y=TSP_Secs]{Figures/avgDegree.txt}; \label{secs_tsp}
            \addplot[mark=o, smooth, color=violet] table[x=avgDegree, y=ESP_Secs]{Figures/avgDegree.txt}; \label{secs_esp}
            
            \nextgroupplot[xlabel=Measured Variables]
            \addplot[mark=square, smooth, color=olive] table[x=numMeasures, y=GRaSP_Secs]{Figures/numMeasures.txt};
            \addplot[mark=x, smooth, color=orange] table[x=numMeasures, y=PC_Secs]{Figures/numMeasures.txt};
            \addplot[mark=+, smooth, color=purple] table[x=numMeasures, y=fGES_Secs]{Figures/numMeasures.txt};
            \addplot[mark=triangle, smooth, color=teal] table[x=numMeasures, y=TSP_Secs]{Figures/numMeasures.txt};
            \addplot[mark=o, smooth, color=violet] table[x=numMeasures, y=ESP_Secs]{Figures/numMeasures.txt};
            
        \end{groupplot}
        \node[fill=white, draw=black] at (3.2, -1.5) {
            \small
            \begin{tabular}{c c c c c c}
                GRaSP$_2$ & \ref{secs_grasp} & GRaSP$_1$ & \ref{secs_esp} & GRaSP$_0$ & \ref{secs_tsp} \\
                fGES & \ref{secs_fges} & PC & \ref{secs_pc}
            \end{tabular}
        };
    \end{tikzpicture}
    \caption{Measured variables fixed to 60 when not varied, average degree fixed to 6 when not varied, sample size fixed to 1,000.}
    \label{fig:secs}
\end{figure}
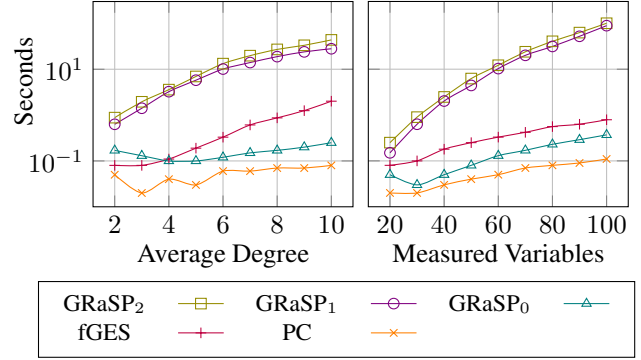

Figure \ref{fig:secs} shows that all the algorithms on average return in under two minutes for the studied scenarios. However, given the log scale, it should be noted that the computation time for GRaSP$_2$ increases exponentially with respect to the average degree of the graph and with respect to the number of measured variables. Other algorithms see similar slow-downs, but, other than GRaSP$_1$, none of the other algorithms experience as significant of a slow-down.\footnote{All simulations in this paper were run on a MacBook Pro laptop computer, M1, 2020, with 16G of RAM, using the Corretto 18 Java SDK. Memory is the main resource constraint on the procedure, which is needed for caching scores. Thanks to the comment of an anonymous reviewer, a machine with 256GB of RAM may be useful for analyses significantly larger than the ones studied.}

In this paper, we focused on algorithms that can run on a 100 variable problem in a reasonable amount of time on a laptop. However, we would be remiss if we did not mention a recent algorithm by \cite{lu2021improving} called Triplet A$^*$ that performs in terms of accuracy as well as, if not better than, GRaSP$_2$. We declined to directly compare the Triple A$^*$ algorithm in our Figures because it was unable to finish our simulations in reasonable time; for instance, the point they give in their Figure 6 for the 60-variable, average degree 5 case was already as slow as could be managed (personal communication); we took our simulations out to an average degree of 10. In lieu of this, we include in Appendix \textcolor{blue}{\textbf{\ref{luetal-comparison}}} results of running GRaSP$_2$ on their published simulation data. 


\section{Empirical Example}
\label{sec:empirical}
We give a simple empirical example, the 6-variable Airfoil example from the Irvine Machine Learning Repository (\cite{Dua:2019}. The experiment measures sound pressure elicited by an airfoil in a wind tunnel. The variables in the example are as follows: (1) Velocity of the wind in the tunnel, (2) chord length of the airfoil, (3) angle of attack of the airfoil, (4) displacement of the wind away from the airfoil, (5) frequency of the elicited sound, and (6) measured pressure of the elicited sound. (1), (2), and (3) are experimental variables and thus exogenous; (6) in the experiment is endogenous. The GRaSP$_2$, PC and GES graphs are given in Appendix \textcolor{blue}{\textbf{\ref{airfoil-example}}}. The GRaSP$_2$ model (which is the same as the SP model) is uniquely frugal; background knowledge is satisfied, except possibly for (3), which looks to be not exogenous in the model; here, it helps to remember that latent variables might exist. This raises the question as to whether a causally insufficient algorithm might find a model consistent with (3) being exogenous. We will explore how GRaSP$_2$ may be used to do latent variable reasoning to see whether (3) remains non-exogenous in general. 

This example has a number of advantages: (a) It is an experiment so readily interpretable as a causal system; (b) because it is an experiment, partial ground truth for the system can easily be adduced; and, (c) it is small enough to run SP on the data, and since this produces a single model, we can simply compare the output of GRaSP$_2$ to the output of SP to show that GRaSP$_2$ finds the optimal BIC model.

Further empirical examples with SP (where possible), GRaSP$_2$, fGES, and PC are given on our GitHub site.


\section{Discussion}
\label{sec:discussion}
Permutation-based reasoning in designing causal search algorithms is increasingly influential in the literature, including the methods from  \cite{teyssier2005ordering} and \cite{raskutti2018learning}. We propose a class of algorithms under the generic name GRaSP characterized by an efficient permutation-based operation, \textit{tuck}. All tiers of GRaSP are shown to be correct and pointwise consistent under the assumption of faithfulness. Also, we show that the two lower tiers of GRaSP are logically equivalent to the algorithms TSP and ESP discussed in \citep{solus2021consistency}. We further prove that the final tier of GRaSP makes a strictly weaker assumption than its lower-tier counterparts and demonstrate that it outperforms the lower-tier algorithms and two standard causal search algorithms, PC and fGES, in simulations.

Discussion of GRaSP can be extended in several directions. First, we have already begun to explore even higher tiers of GRaSP which relax the search criterion even further. Figure \ref{fig:avgDeg} suggests that GRaSP may provide tools helpful for the discussion of dense graph search. Given the hierarchy of GRaSP, higher tiers will hopefully improve the performance statistics and employ weaker assumption than the existing tiers. Ultimately, we hope to develop a tier of GRaSP that is correct under u-frugality alone.

Second, many advances have been made in the area of more or completely general modeling of data distributions, with corresponding improvements in accuracy of causal search for algorithms taking general modeling assumptions into account. It would be helpful to consider how such ideas can be incorporated into GRaSP. For example, \cite{huang2018generalized} show how a consistent general score can be incorporated into GES; it will be interesting to see whether GRaSP is able to show similar improvement in applicability when using such a score.

Third, we have analyzed Gaussian simulations in Section \textcolor{blue}{\textbf{\ref{sec:sims}}}, but some simulation work needs to be done to show that GRaSP works well for discrete distributions (where the theory is already applicable) and also for mixed Gaussian/discrete distributions studied in \citep{andrews2019learning}.

Fourth, the discussion of this paper is built upon the assumptions of causal sufficiency, that is, no latent common causes, and no selection bias. Causal search without these assumptions was pioneered by the FCI algorithm from \cite{spirtes2000causation} and \cite{ZHANG20081873}. To improve empirical performance of FCI, \cite{ogarrio2016hybrid} initiated a hybrid algorithm GFCI which combines GES with FCI. To follow suit, we plan to explore an algorithm that incorporates GRaSP into GFCI (in place of GES), further improving this empirical performance.

Fifth, more direct comparisons to other algorithms need ideally to be done. As a step in this direction, we include figures on our GitHub site using the simulation  parameters in \citep{lu2021improving}, corresponding to their Figures 6, so there is oblique comparison to the algorithms in those figures, including GES and PC in the PCALG package \citep{kalisch2012causal}, Triplet A$^*$ \citep{lu2021improving}, NOTEARS \citep{zheng2018dags}, the GSP implementation in the Python causaldag package, LiNGAM \citep{shimizu2006linear}, and MMHC \citep{tsamardinos2006max}. The reader is invited to explore those comparisons.

Finally, we have taken up just one real data example in this paper, but it is useful to point out in a forward-looking way that improvements in the ability to handle latent and mixed continuous/discrete variables in a scalable and accurate causal search algorithm would put one in a good position to analyze a number of otherwise difficult real data examples. Accurate preliminary results consistent with ground truth using the suggested modification of GFCI for a number of mixed datasets from the Irvine Machine Learning Repository (\citep{Dua:2019}), for instance, suggest that this would be a good direction to look for new practical methods (cf. \citep{raghu2018comparison}).

\begin{contributions}
    WL contributed theoretical results, with input from BA, while BA and JR worked on the algorithm implementations and contributed empirical results. All authors contributed to algorithmic development.
\end{contributions}

\begin{acknowledgements}
    We thank Greg Cooper, Clark Glymour, Ignavier Ng, Peter Spirtes, Jiji Zhang, and Kun Zhang for discussion and feedback, and the anonymous reviewers for detailed and insightful comments.
\end{acknowledgements}

\bibliography{refs}
\clearpage
\onecolumn
\appendix

\section{Background materials}
\label{app:bkgd}
\subsection{Graphical Definitions}
\label{app:graph_def}

A \textit{directed graph} $\G$ over a set of measured variables $\mb{V} = \{X_1,...,X_m\}$ consists of $m$ vertices $\mb{v} = \{1,...,m\}$ where each vertex $i \in \mb{v}$ associates to a variable $X_i \in \mb{V}$, and each edge in $\G$ is directed with the form $j \to k$ and no vertex has a directed edge to itself. A \textit{directed path} $\mf{p}$ is a sequence of vertices $\la \mb{i}_1, \mb{i}_2,..., \mb{i}_k\ra$ for some $k \geq 2$ where $\mb{i}_j \in \mb{v}$ for each $1 \leq j \leq k$, and $\mb{i}_j$ and $\mb{i}_{j+1}$ are connected by a directed edge (i.e., $\mb{i}_j \to \mb{i}_{j+1}$ or $\mb{i}_{j+1} \to \mb{i}_{j}$). Such a path $\mf{p}$ is \textit{unidirectional} if $\mb{i}_{j} \to \mb{i}_{j+1}$ for each $1 \leq j < k$. A \textit{directed acyclic graph} (DAG) is a directed graph where no vertex can have a unidirectional directed path to itself.

Denote $\E(\G)$ as the set of directed edges in $\G$. A pair of DAGs $\G_1, \G_2$ over the same set of variables $\mb{V}$ are equivalent if and only if $\E(\G_1) = \E(\G_2)$. Let $\Pa(j, \G) = \{k \in \mb{v}: (k \to j) \in \E(\G)\}$ be the set of \textit{parents} of $j$ in $\G$, and $\Ch(j, \G) = \{k \in \mb{v}: (j \to k) \in \E(\G)\}$ be the set of \textit{children} of $j$ in $\G$. $\An(j, \G)$, the \textit{ancestors} of $j$ in $\G$, is defined by the transitive closure of $\Pa(j, \G)$. Similarly, $\De(j, \G)$, the \textit{descendants} of $j$ in $\G$, is defined by the transitive closure of $\Ch(j, \G)$ and union with $\{j\}$ itself (i.e., $j$ is its own descendant). Further let $\Nd(j, \G) = \mb{v} \setminus \De(j, \G)$ be the set of $j$'s \textit{non-descendants}.

A pair of vertices $j, k \in \mb{v}$ are said to be \textit{adjacent} in $\G$ if $(j \to k) \in \E(\G)$ or $(k \to j) \in \E(\G)$. For any triple of pairwise distinct vertices $i, j, k \in \mb{v}$, we say that $(i, j, k)$ is \textit{unshielded} if $(i, j)$ and $(j, k)$ are adjacent pairs in $\G$, but not $(i, k)$. $(i, j, k)$ forms a \textit{triangle} if they are pairwise adjacent. If $(i, j, k)$ is an unshielded triple or is a triangle, $j$ is a \textit{collider} (on the directed path $\la i, j, k\ra$) if $(i \to j), (k \to j) \in \E(\G)$, and a \textit{non-collider} otherwise. A directed path $\mf{p}$ is a \textit{trek} if it contains no collider. 

For any $j, k \in \mb{v}$ and any $\mb{i} \subseteq \mb{v} \setminus \{j, k\}$, $j$ and $k$ are \textit{d-connected} given $\mb{i}$ in $\G$ if there exists a directed path $\mf{p}$ between $j$ and $k$ in $\G$ such that no non-collider on $\mf{p}$ is in $\mb{i}$, and each collider $l$ on $\mf{p}$ or a $l$'s descendant is in $\mb{i}$. $j$ and $k$ are \textit{d-separated} given $\mb{i}$ in $\G$ if $j$ and $k$ are not d-connected given $\mb{i}$. For any disjoint subsets of vertices $\mb{j}, \mb{k}, \mb{i} \subseteq \mb{v}$, $\mb{j}$ and $\mb{k}$ are d-separated given $\mb{i}$ in $\G$ if $j$ and $k$ are d-separated by $\mb{i}$ in $\G$ for every $j \in \mb{j}$ and every $k \in \mb{k}$. 

Given a model $(\G, \Prob)$ over $\mb{V}$, $\G$ is said to be \textit{local Markov} to $\Prob$ if $X_j \CI_\Prob \mb{X}_{\Nd(j, \G)} \setminus \mb{X}_{\Pa(j, \G)}\,|\,\mb{X}_{\Pa(j, \G)}$ for every $j \in \mb{v}$. It is a well-known fact that $\G$ is local Markov to $\Prob$ if and only if $\I(\G) \subseteq \I(\Prob)$ (i.e., global Markov as defined by d-separation).

\subsection{Graphoid Axioms} 
\label{app:graphoid}

For any pairwise disjoint sets of variables $\mb{W}, \mb{X}, \mb{Y},$ and $\mb{Z}$,
\begin{align*}
\mb{X} \CI \mb{Y} \,|\,\mb{Z} &\,\,\Rightarrow\,\, \mb{Y} \CI \mb{X} \,|\,\mb{Z} & (\textit{symmetry})\\
\mb{X} \CI \mb{Y} \cup \mb{W} \,|\,\mb{Z} &\,\,\Rightarrow\,\, (\mb{X} \CI \mb{Y} \,|\,\mb{Z})  \wedge (\mb{X} \CI \mb{W}\,|\,\mb{Z}) &
(\textit{decomposition})\\
\mb{X} \CI \mb{Y} \cup \mb{W} \,|\,\mb{Z} &\,\,\Rightarrow\,\, \mb{X} \CI \mb{Y} \,|\,\mb{Z} \cup \mb{W} &
(\textit{weak union})\\
(\mb{X} \CI \mb{Y}\,|\,\mb{Z}) \wedge (\mb{X} \CI \mb{W}\,|\,\mb{Z} \cup \mb{Y}) &\,\,\Rightarrow\,\, \mb{X} \CI \mb{Y}\cup \mb{W}\,|\,\mb{Z} &(\textit{contraction})\\
(\mb{X} \CI \mb{Y}\,|\,\mb{Z} \cup \mb{W}) \wedge (\mb{X} \CI \mb{W}\,|\,\mb{Z} \cup \mb{Y}) &\,\,\Rightarrow\,\, \mb{X} \CI \mb{Y}\cup \mb{W}\,|\,\mb{Z} &(\textit{intersection})\\
(\mb{X} \CI \mb{Y}\,|\,\mb{Z}) \wedge (\mb{X} \CI \mb{W}\,|\,\mb{Z}) &\,\,\Rightarrow\,\, \mb{X} \CI \mb{Y}\cup \mb{W}\,|\,\mb{Z} & (\textit{composition})
\end{align*}




A distribution $\Prob$ is a \textit{semigraphoid} if $\I(\Prob)$ is closed under \textit{symmetry}, \textit{decomposition}, \textit{weak union}, and \textit{contraction}. A semigraphoid $\Prob$ is a \textit{graphoid} if $\I(\Prob)$ is closed under \textit{intersection}. A graphoid $\Prob$ is \textit{compositional} if $\I(\Prob)$ is closed under \textit{composition}. See Chapter 2 of \citep{Studeny10.5555/1044858} for a more comprehensive study of graphoid axioms. In addition, applications of \textit{symmetry} in our upcoming proofs will be done implicitly for the sake of simplicity. 

Additionally, \cite{Spohn1994}
notes that the following property necessarily holds in the independence models induced by positive discrete probability distributions. For any pairwise disjoint sets of variables $\mb{W}, \mb{X}, \mb{Y},$ and $\mb{Z}$,
\[
    (\mb{X} \CI \mb{Y}\,|\,\mb{W} \cup \mb{Z}) \wedge (\mb{W} \CI \mb{Z}\,|\,\mb{X} \cup \mb{Y}) \wedge (\mb{W} \CI \mb{Z}\,|\,\mb{X}) \,\,\Rightarrow\,\, [(\mb{W} \CI \mb{Z}\,|\,\mb{Y}) \Leftrightarrow (\mb{W} \CI \mb{Z}\,|\,\varnothing)] \qquad (\textit{Spohn condition})
\]

\subsection{DAG induced from a permutation}
\label{app:DAG_induce}

\begin{definition}
\label{MB}
Given a semigraphoid $\Prob$ over $\mb{V}$, for every $X \in \mb{V}$, we say that $\mb{M} \subseteq \mb{V}$ is a \textit{Markov blanket} of $X$ relative to $\mb{Z} \subseteq \mb{V} \setminus \{X\}$ if
\begin{enumerate}
    \item[(i)] $\mb{M} \subseteq \mb{Z}$;
    \item[(ii)] $X \CI_\Prob (\mb{Z}\setminus \mb{M})\,|\,\mb{M}$.
\end{enumerate}
Such a Markov blanket $\mb{M}$ is said to be a \textit{Markov boundary} if it further satisfies the following condition:
\begin{enumerate}
    \item[(iii)] there does not exist $\mathbf{M}' \subset \mathbf{M}$ s.t. $X \CI_\Prob (\mb{Z} \setminus \mb{M}')\,|\,\mb{M}'$.
\end{enumerate}
\end{definition}

\begin{lemma}
\label{MB_unique}
\citep{verma1988causal} Given a graphoid $\Prob$ over $\mb{V}$, for every $X \in \mb{V}$ and every $\mb{Z} \subseteq \mb{V} \setminus \{X\}$, there is a \textit{unique} Markov boundary of $X$ relative to $\mb{Z}$.
\end{lemma}

In the following, we use $\MB_\Prob(X, \mb{Z})$ to refer to the unique Markov boundary of $X$ relative to $\mb{Z}$. The subscript $\Prob$ will be suppressed if the underlying graphoid is clear from context.

\begin{lemma}
\label{MB_subset}
Given a graphoid $\Prob$  over $\mb{V}$, for every $X \in \mb{V}$ and every $\mb{Z} \subseteq \mb{V} \setminus \{X\}$, if $\mb{M}$ is a Markov blanket of $X$ relative to $\mb{Z}$, then $\MB(X, \mb{Z}) \subseteq \mb{M}$.
\end{lemma}
\begin{proof}
Immediate from \textbf{Definition \ref{MB}} and \textbf{Lemma \ref{MB_unique}}.
\end{proof}

Next, we revisit the two methods of inducing a DAG from a permutation. Given a semigraphoid $\Prob$ over $\mb{V}$, each $\pi \in \Pi(\mb{v})$ induces a DAG satisfying the following condition:
\begin{align}\tag{VP}
X_j \in \mb{M} \Leftrightarrow (j \to k) \in \E(\G_\pi)
\end{align}
where $\mb{M}$ is a Markov boundary of $X_k$ relative to $\mb{X}_{\Pre(k, \pi)}$. (VP) is the construction of a \textit{boundary DAG} in \citep{verma1988causal}. On the other hand, given a graphoid $\Prob$ over $\mb{V}$, each $\pi \in \Pi(\mb{v})$ induces a DAG satisfying the following condition:
\begin{align}\tag{RU}
    j \in \Pre(k, \pi) \text{ and } X_j \nCI_\Prob X_k\,|\,\mb{X}_{\Pre(k, \pi)\setminus \{j\}} \Leftrightarrow (j \to k) \in \E(\G_\pi). 
\end{align}

We want to show that the two DAG-inducing methods are equivalent when the underlying distribution is a graphoid.

\begin{lemma}
\label{VP=RU}
Given a graphoid $\Prob$ over $\mb{V}$, consider any $\pi \in \Pi(\mb{v})$. Let $\G_\pi$ be the DAG induced from $\pi$ by (VP), and $\mc{H}_\pi$ be the DAG induced from $\pi$ by (RU). Then $\G_\pi = \mc{H}_\pi$. 
\end{lemma}
\begin{proof}
We divide the proof into two directions: (VP) $\Rightarrow$ (RU), and (VP) $\Leftarrow$ (RU). Consider any $j, k \in \mb{v}$ where $\pi[j] < \pi[k]$ such that $j \in \Pre(k, \pi)$. Let $\mathbf{M}$ be the unique Markov boundary $\MB(X_k, \mb{X}_{\Pre(k, \pi)})$.

[$\Rightarrow$] Suppose that $(j \to k) \notin \E(\G_\pi)$. We have $X_j \notin \mb{M}$. By \textbf{Definition \ref{MB}} (ii), we then have,
\begin{align}
    X_k \CI_\Prob& ((\mb{X}_{\Pre(k, \pi)} \setminus \mb{M}) \setminus \{X_j\}) \cup \{X_j\}\,|\,\mb{M} &\because X_k \CI_\Prob \mb{X}_{\Pre(k, \pi)} \setminus \mb{M}\,|\,\mb{M}\\
    X_k \CI_\Prob& X_j\,|\,\mb{M} \cup ((\mb{X}_{\Pre(k, \pi)} \setminus \mb{M}) \setminus \{X_j\}) &\because (1), \text{weak union}\\
    X_k \CI_\Prob& X_j\,|\,\mb{X}_{\Pre(k, \pi) \setminus \{j\}} &\because (2)
\end{align}
where the last formula amounts to $(j \to k) \notin \E(\mc{H}_\pi)$ by (RU).

[$\Leftarrow$] Suppose that $(j \to k) \notin \E(\mc{H}_\pi)$. We have $X_k \CI_\Prob X_j\,|\,\mb{X}_{\Pre(k, \pi)\setminus \{j\}}$. Let $\mb{M}'$ be $\mb{X}_{\Pre(k, \pi) \setminus \{j\}}$. We have $X_k \CI_\Prob (\mb{X}_{\Pre(k, \pi)} \setminus \mb{M}')\,|\,\mb{M}'$ such that $\mb{M}'$ is a Markov blanket of $X_k$ relative to $\mb{X}_{\Pre(k, \pi)}$. By \textbf{Lemma \ref{MB_subset}}, $X_j \notin \mb{M} \subseteq \mb{M}'$ and therefore $(j \to k) \notin \E(\G_\pi)$ by (VP). 
\end{proof}

\begin{theorem}
\label{VP-theorem}
\citep{Pearl_10.5555/534975} Given a semigraphoid $\Prob$ over $\mb{V}$, $\G_\pi$ induced by $\pi$ using (VP) is Markovian and SGS-minimal for any $\pi \in \Pi(\mb{v})$. 
\end{theorem}

\textbf{Theorem \ref{RU-theorem}}
\hspace{0.2cm} \textit{Given a graphoid $\Prob$ over $\mb{V}$, $\G_\pi$ induced by $\pi$ using (RU) is Markovian and SGS-minimal for any $\pi \in \Pi(\mb{v})$.}\\

\begin{proof}
Immediate from \textbf{Lemma \ref{VP=RU}} and \textbf{Theorem \ref{VP-theorem}}.\footnote{The two DAG-inducing methods were not differentiated in \citep{raskutti2018learning}. Thus, we provide a proof of \textbf{Theorem \ref{RU-theorem}}.}
\end{proof}


\section{Correctness results}
\label{app:correct}
First, we introduce some permutation-based notations to facilitate our coming proofs. In this section, we use $\G_\pi$ to denote the DAG induced by $\pi$ from a graphoid $\Prob$ using (RU) unless specified otherwise.

Given a set of variables $\mb{V}$, consider any $\pi \in \Pi(\mb{v})$ and any pair $j, k \in \mb{v}$ where $\pi[j] < \pi[k]$. $\pi$ can be written as $\la \bs{\delta}_{<j}, j, \bs{\delta}_{j\sim k}, k, \bs{\delta}_{>k}\ra$ such that $\bs{\delta}_{<j} = \la \pi_i: 1 \leq i < \pi[j] \ra$, $\bs{\delta}_{j \sim k} = \la \pi_i: \pi[j] < i < \pi[k]\ra$, and $\bs{\delta}_{>k} = \la \pi_i: \pi[k] < i \leq |\pi|\ra$. When $\bs{\delta}_{j \sim k} = \varnothing$, we say that $j$ and $k$ are $\pi$-\textit{adjacent}. In that case, $\pi$ can be written as $\la \bs{\delta}_{<j}, j, k, \bs{\delta}_{>k}\ra$ instead.

\begin{definition}
\label{AT}
Given a set of variables $\mb{V}$, for any $\pi, \tau \in \Pi(\mathbf{v})$, 
\begin{enumerate}
    \item[(a)] $\tau$ is said to be $(j, k)$-different from $\pi$ for some $j, k \in \mathbf{v}$ if $j$ and $k$ are $\pi$-adjacent (i.e., $\pi = \la \bs{\delta}_{<j}, j, k, \bs{\delta}_{>k}\ra$) and $\tau = \la \bs{\delta}_{<j}, k, j, \bs{\delta}_{>k} \ra$;
    \item[(b)] $\pi$ and $\tau$ are said to be in adjacent transposition (AT) if they are $(j, k)$-different for some $j, k \in \mathbf{v}$.
\end{enumerate}
\end{definition}

\begin{lemma}
\label{SGS_same_DAG}
Given a graphoid $\Prob$ over $\mb{V}$, consider any $\mc{H} \in \CMC(\Prob)$. If $\pi \in \Pi(\mb{v})$ is a causal order of $\G$, then $\E(\G_\pi) \subseteq \E(\mc{H})$. Also, $\G_\pi = \mc{H}$ if $\mc{H} \in \SGS(\Prob)$. 
\end{lemma}
\begin{proof}
Consider any $k \in \mathbf{v}$ and $\mt{Nd}(k, \mc{H})$ (i.e., the set of $k$'s non-descendants in $\mc{H}$). Since $\mc{H} \in \mathtt{CMC}(\mathcal{P})$, it follows that $X_k \CI_\Prob \mb{X}_{\mt{Nd}(k, \mc{H})} \setminus \mb{X}_{\Pa(k, \mc{H})} \,|\,\mb{X}_{\Pa(k, \mc{H})}$. Also, we have $\Pa(k, \mc{H}) \subseteq \Pre(k, \pi) \subseteq \mt{Nd}(k, \mc{H})$ from $\pi$'s being a causal order of $\mc{H}$. By \textit{decomposition}, we have
$X_k \CI_\Prob \mb{X}_{\Pre(k, \pi)} \setminus \mb{X}_{\Pa(k, \mc{H})} \,|\,\mb{X}_{\Pa(k, \mc{H})}$ such that $\mb{X}_{\Pa(k, \mc{H})}$ is a Markov blanket of $X_k$ relative to $\mb{X}_{\Pre(k, \pi)}$. By \textbf{Lemma \ref{MB_subset}}, we have $\mathtt{MB}(X_k, \mb{X}_{\mathtt{Pre}(k, \pi)}) \subseteq \mathtt{Pa}(k, \mc{H})$. Consider $\G_\pi$ induced by (VP). The above entails that $\E(\G_\pi) \subseteq \E(\mc{H})$ since $\Pa(k, \G_\pi) = \MB(X_k, \mb{X}_{\Pre(k, \pi)}) \subseteq \Pa(k, \mc{H})$ for each $k \in \mb{v}$. Due to \textbf{Lemma \ref{VP=RU}}, $\E(\G_\pi) \subseteq \E(\mc{H})$ still holds even if $\G_\pi$ is induced by (RU). Lastly, $\G_\pi = \mc{H}$ follows from \textbf{Definition \ref{SGS-minimal}} if $\mc{H} \in \SGS(\Prob)$.
\end{proof}

\begin{lemma}
\label{AT_iff}
\citep{solus2021consistency} Given a graphoid $\Prob$ over $\mb{V}$, consider any $\pi, \tau \in \Pi(\mb{v})$ where $\tau$ is $(j, k)$-different from $\pi$ for some $j, k \in \mb{v}$. Then $\G_\pi = \G_\tau$ if and only if $X_j \CI_\Prob X_k\,|\,\mb{X}_{\Pre(j, \pi)}$.
\end{lemma}

\begin{proof}
Suppose that $X_j \nCI_\Prob X_k\,|\,\mb{X}_{\Pre(j, \pi)}$. By (RU), we have $(j \to k) \in \E(\G_\pi)$. Note that $(j \to k) \notin \E(\G_\tau)$ since $\tau[k] < \tau[j]$ and $\tau$ is a causal order of $\G_\tau$ by construction. Hence, $\G_\pi \neq \G_\tau$. 

On the other hand, suppose that $X_j \CI_\Prob X_k\,|\,\mb{X}_{\Pre(j, \pi)}$. Since $\tau$ is $(j, k)$-different from $\pi$, we have $\pi = \la \bs{\delta}_{<j}, j, k, \bs{\delta}_{>k}\ra$ and $\tau = \la \bs{\delta}_{<j}, k, j, \bs{\delta}_{>k}\ra$ according to \textbf{Definition \ref{AT}} (a). By (RU), we know that $(k \to j) \notin \E(\G_\tau)$. Hence, $\pi$ is a causal order of $\G_\tau$. By \textbf{Theorem \ref{RU-theorem}}, $\G_\tau \in \SGS(\Prob)$. Therefore, it follows from \textbf{Lemma \ref{SGS_same_DAG}} that $\G_\tau = \G_\pi$.
\end{proof} 

\begin{lemma}
\label{covered_AT}
Given a graphoid $\Prob$ over $\mathbf{V}$, consider any $\pi \in \Pi(\mathbf{v})$. Suppose that $\G_\pi$ contains a covered edge $j \to k$ where $\pi = \la \bs{\delta}_{<j}, j, \bs{\delta}_{j \sim k}, k, \bs{\delta}_{>k}\ra$. If $\tau = \la \bs{\delta}_{<j}, j, k, \bs{\delta}_{j \sim k}, \bs{\delta}_{>k}\ra$, then $\G_\pi = \G_\tau$.
\end{lemma}
\begin{proof}
Since $j \to k$ is a covered edge in $\G_\pi$, it follows that $(i \to k) \notin \E(\G_\pi)$ for each $i \in \bs{\delta}_{j\sim k}$, and thus $X_i \CI_\Prob X_k\,|\,\mb{X}_{\Pre(i, \pi)}$ by (RU). Hence, $\G_\pi = \G_\tau$ can be obtained after $|\bs{\delta}_{j \sim k}|$ applications of \textbf{Lemma \ref{AT_iff}}.
\end{proof}

\begin{theorem}
\label{Jiji_thm}
\citep{zhang2013comparison} Given a set of variables $\mb{V}$, for any $\G, \mc{H} \in \mt{DAG}(\mb{V})$, if $\E(\G) \subseteq \E(\mc{H})$, then $\I(\mc{H}) \subseteq \I(\G)$. 
\end{theorem}

\begin{lemma}
\label{covered_MEC}
\citep{10.5555/2074158.2074169} Consider any DAG $\G$. Let $\mc{H}$ be the result of reversing $(i \to j) \in \E(\G)$. Then $\mc{H} \in \MEC(\G)$ if and only if $i \to j$ is a covered edge. 
\end{lemma}

\begin{theorem}
\label{covered_MEC_chain}
\citep{10.5555/2074158.2074169} Consider any pair of DAGs $\G$ and $\mc{H}$ over the same set of variables s.t. $\mc{H} \in \MEC(\G)$, and for which there are $k$ edges in $\G$ that have opposite orientation in $\mc{H}$. Then there exists a sequence of $k$ distinct covered edge reversals in $\G$ s.t. $\G$ becomes $\mc{H}$ after all reversals. 
\end{theorem}

\begin{lemma}
\label{tuck_lemma}
Given a graphoid $\Prob$ over $\mb{V}$, consider any $\pi \in \Pi(\mb{v})$. Suppose that $(j \to k) \in \E(\G_\pi)$ is a covered edge, and let $\mc{H}$ be the DAG resulted from reversing $(j \to k)$ in $\G_\pi$. If $\tau = \textit{tuck}(\pi, j, k)$, then
\begin{enumerate}
    \item[(a)] $\tau$ is a causal order of $\mc{H}$;
    \item[(b)] $\E(\G_\tau) \subseteq \E(\mc{H})$;
    \item[(c)] $|\E(\G_\tau)| \leq |\E(\G_\pi)|$;
    \item[(d)] $\I(\G_\pi) \subseteq \I(\G_\tau)$. 
\end{enumerate}
\end{lemma}
\begin{proof}
A similar lemma has been shown in \citep{solus2021consistency}. First, we write $\pi = \la \bs{\delta}_{<j}, j, \bs{\delta}_{j \sim k}, k, \bs{\delta}_{>k} \ra$ as usual. Consider $\pi' = \la \bs{\delta}_{<j}, j, k, \bs{\delta}_{j \sim k}, \bs{\delta}_{>k} \ra$. By \textbf{Lemma \ref{covered_AT}}, we have $\G_\pi = \G_{\pi'}$. Note that $\tau = \textit{tuck}(\pi, j, k) = \la \bs{\delta}_{<j}, k, j, \bs{\delta}_{j \sim k}, \bs{\delta}_{>k}\ra$ because $\textit{tuck}(\pi, j, k)$ is a covered tuck. Thus, $\tau$ is $(j, k)$-different from $\pi'$. Also, since $\pi'$ is a causal order of $\G_\pi$, it follows that $\tau$ is a causal order of $\mc{H}$ and thus (a) is proven.

Next, observe that $\I(\G_\pi) = \I(\mc{H})$ from \textbf{Lemma \ref{covered_MEC}}. From $\G_\pi \in \CMC(\Prob)$ by \textbf{Theorem \ref{RU-theorem}}, we know that $\mc{H} \in \CMC(\Prob)$. Thus, (b) immediately follows from (a) and \textbf{Lemma \ref{SGS_same_DAG}}. Also, (c) is entailed by $|\E(\G_\tau)| \leq |\E(\mc{H})| = |\E(\G_\tau)|$. Finally, by \textbf{Theorem \ref{Jiji_thm}}, we have $\I(\G_\pi) = \I(\mc{H}) \subseteq \I(\G_\tau)$ as desired in (d).
\end{proof}

Before we compare TSP and unbounded GRaSP$_0$, we want to make an assumption related to how the set of covered edges in any particular DAG is ordered. To see the importance of such an assumption, observe that different orderings of $\E^0(\G_\pi)$ (i.e., the set of covered edges in an induced DAG $\G_\pi$) can alter the output of TSP and also GRaSP$_0$. For example, suppose that $(j \to k), (j' \to k') \in \E^0(\G_\pi)$. Say the DFS of GRaSP$_0$ starts with performing $\textit{tuck}(\pi, j, k)$ and leads to some permutation $\tau$. However, choosing to perform $\textit{tuck}(\pi, j', k')$ instead at the beginning of the DFS procedure can lead to some $\tau'$ where $\G_{\tau} \neq \G_{\tau'}$. Hence, we enforce the assumption that the ordering of $\E^0(\G)$ for any DAG $\G$ is fixed arbitrarily. For instance, $(j \to k)$ precedes $(j' \to k')$ in $\E^0(\G)$ if $j < j'$, or $j = j'$ and $k < k'$. Consequently, the issue of order-dependence can be avoided even when comparing a Chickering sequence found by TSP and a ct-sequence found by unbounded GRaSP$_0$. In the following, this assumption will be made implicitly.  

Now we revisit how TSP works. Given a graphoid $\Prob$ over $\mb{V}$ and an initial permutation $\pi \in \Pi(\mb{v})$, TSP begins with setting $\G$ as the induced $\G_\pi$. Starting with the root $\G$, TSP performs DFS to identify a SGS-minimal DAG $\mc{H}$ connected by a Chickering sequence from $\G$ such that $|\E(\G)| > |\E(\mc{H})|$. TSP returns $\G$ if no such $\mc{H}$ is found. Otherwise, it updates $\G$ as $\mc{H}$ and repeat the procedure.

The DFS procedure of TSP aims to traverse from one SGS-minimal DAG to another SGS-minimal DAG by the construction of a Chickering sequence. Though we know that a Chickering sequence is obtained by the reversals of covered edges and deletions of directed edges, \cite{solus2021consistency} did not specify any ordering of these operations. Below we provide a more precise definition of the Chickering sequences considered by TSP. 

\begin{definition}
\label{TSP_Chickering}
Given a graphoid $\Prob$ over $\mb{V}$, a TSP-Chickering sequence $\mf{C} = \la \G^1,..., \G^m\ra$ is a Chickering sequence satisfying the following condition:
\begin{enumerate}
    \item[(a)] $\G^1, \G^m \in \SGS(\Prob)$;
    \item[(b)] $\G^i$ and $\G^{i'}$ are pairwise distinct for $1 \leq i < i' \leq m$; 
    \item[(c)] if $|\E(\G^1)| = |\E(\G^m)|$, then $\G^1,..., \G^m \in \SGS(\Prob)$ where they differ by the reversals of some covered edges;
    \item[(d)] otherwise, there exists a turning index $1 < l < m$ such that (i) $\G^1,..., \G^{l-1} \in \SGS(\Prob)$, (ii) $\G^1,..., \G^l$ differ by the reversals of some covered edges, and (iii) $\G^{i+1}$ is obtained from deleting a directed edge in $\G^{i} \notin \SGS(\Prob)$ for each $l \leq i < m$. 
\end{enumerate}
\end{definition}

Readers are suggested to find the original pseudocode of TSP in \cite{solus2021consistency} to verify that our \textbf{Definition \ref{TSP_Chickering}} is a fair description of the Chickering sequences considered by TSP. Conditions (a) and (b) are straightforward. (c) refers to the case where TSP cannot find a sparser SGS-minimal DAG. So if any $\G^i$ in $\mf{C}$ were non-SGS-minimal, then TSP would have obtained a proper subgraph of $\G^i$ which is SGS-minimal by a series of edge-deletion. (d) refers to the case where TSP manages to find a sparser SGS-minimal DAG. Notice that $\G^2$ must be obtained by a covered edge reversal from $\G^1$ since $\G^1 \in \SGS(\Prob)$. If $\G^2 \notin \SGS(\Prob)$, then TSP can obtain the desired SGS-minimal DAG by a series of edge-deletion from $\G^2$. But if $\G^2 \in \SGS(\Prob)$, the procedure above repeats until finding the turning index $l$ such that $\G^l \notin \SGS(\Prob)$ and then the sparser $\G^m \in \SGS(\Prob)$ can be obtained by a series of edge-deletion from $\G^l$. 

Now we compare TSP and unbounded GRaSP$_0$ by considering their respective sequences traversed in the DFS procedure.

\begin{lemma}
\label{ct_2_Chickering}
Given a graphoid $\Prob$ over $\mb{V}$, consider any $\pi \in \Pi(\mb{v})$ and $\tau = \textit{tuck}(\pi, j, k)$ where $(j \to k) \in \E^0(\G_\pi)$. Given that $\mf{T} = \la \pi, \tau\ra$ is a ct-sequence,
\begin{itemize}
    \item[(a)] if $|\E(\G_\tau)| = |\E(\G_\pi)|$, then $\mf{C} = \la \G_\pi, \G_\tau\ra$ is a TSP-Chickering sequence where $\G_\tau$ is obtained from reversing $(j \to k) \in \E^0(\G_\pi)$;
    \item[(b)] otherwise, there exists a TSP-Chickering sequence $\mf{C} = \la \G_\pi = \G^1,..., \G^m = \G_\tau\ra$ s.t. $\G^2$ is obtained from reversing $(j \to k) \in \E^0(\G_\pi)$, and $\G^{i+1}$ is obtained from deleting a directed edge in $\G^i$ for each $2 \leq i < m$.
\end{itemize}
\end{lemma}
\begin{proof}
First, consider the DAG $\mc{H}$ obtained from reversing $(j \to k) \in \E^0(\G_\pi)$. We start with the case in (a) where $|\E(\G_\tau)| = |\E(\G_\pi)| = |\E(\mc{H})|$. We want to show that $\G_\tau = \mc{H}$. By \textbf{Lemma \ref{tuck_lemma}} (b), we have $\E(\G_\tau) \subseteq \E(\mc{H})$. If $\E(\G_\tau) \subset \E(\mc{H})$ holds, then $|\E(\G_\pi)| = |\E(\mc{H})|$ will be violated. Hence, we have $\G_\tau = \mc{H}$ and thus $\mf{C} = \la \G_\pi, \mc{H} = \G_\tau \ra$ is our desired TSP-Chickering sequence.

For (b), it follows from \textbf{Lemma \ref{tuck_lemma}} (c) that $|\E(\G_\tau)| < |\E(\mc{H})| = |\E(\G_\pi)|$. Let $\G_\pi$ and $\mc{H}$ be $\G^1$ and $\G^2$ respectively. By \textbf{Lemma \ref{tuck_lemma}} (b) again, we have $\E(\G_\tau) \subset \E(\G^2)$ such that we can remove a directed edge from $\G^2$ once at a time until obtaining $\G_\tau$. Therefore, we have the desired TSP-Chickering sequence in (b).
\end{proof}

\begin{lemma}
\label{Chickering_2_ct}
Given a graphoid $\Prob$ over $\mb{V}$, consider any TSP-Chickering sequence $\mf{C} = \la \G^1,...,\G^m\ra$. Let $\pi^1$ be a causal order of $\G^1$. Then
\begin{itemize}
    \item[(a)] if $|\E(\G^1)| = |\E(\G^m)|$, then $\G^{i+1} = \G_{\pi^{i+1}} = \G_{\textit{tuck}(\pi, j, k)}$ where $j \to k$ is the covered edge reversed to obtain $\G^{i+1}$ from $\G^i$ for each $1 \leq i < m$ s.t. $\mf{T} = \la \pi^1,...,\pi^m\ra$ is a ct-sequence;
    \item[(b)] otherwise, then $\G^{i+1} = \G_{\pi^{i+1}} = \G_{\textit{tuck}(\pi, j, k)}$ where $j \to k$ is the covered edge reversed to obtain $\G^{i+1}$ from $\G^i$ for each $1 \leq i < l$ where $l$ is the turning index of $\mf{C}$ and $\G_{\pi^l} = \G^m$ s.t. $\mf{T} = \la \pi^1,...,\pi^l\ra$ is a ct-sequence.
\end{itemize}
\end{lemma}
\begin{proof}
(a) can be easily shown by \textbf{Lemma \ref{tuck_lemma}}(a) and \textbf{Lemma \ref{SGS_same_DAG}}. For (b), the proof of $\G^i = \G_{\pi^i}$ for each $1 \leq i < l$ is similar to that in (a). So we consider $l$ where $\G^l \notin \SGS(\Prob)$ according to \textbf{Definition \ref{TSP_Chickering}}(d). However, it follows from \textbf{Lemma \ref{tuck_lemma}} (a) that $\pi^l$ is a causal order of $\G^l$. Since $\E(\G^m) \subset \E(\G^l)$, we know that $\pi^l$ is also a causal order of $\G^m$. Lastly, given that $\G^m \in \SGS(\Prob)$, it follows from \textbf{Lemma \ref{SGS_same_DAG}} that $\G_{\pi^l} = \G^m$. 
\end{proof} 

\,\\
\textbf{Lemma \ref{ct-better}} \hspace{0.1cm} \textit{Given a graphoid $\Prob$, for any $\pi \in \Pi(\mb{v})$ and any Chickering sequence from $\G_\pi$ to some $\mc{H} \in \SGS(\Prob)$ considered by TSP, there exists a ct-sequence $\la \pi,...,\tau\ra$ s.t. $\G_\tau = \mc{H}$.}\\

\begin{proof}
Given that a Chickering sequence considered by TSP is simply a TSP-Chickering sequence defined in \textbf{Definition \ref{TSP_Chickering}}, the lemma follows immediately from \textbf{Lemma \ref{Chickering_2_ct}}.
\end{proof}

\,\\
\textbf{Theorem \ref{TSP=GRaSP0}} \hspace{0.1cm} \textit{Given a graphoid $\Prob$ over $\mb{V}$ and any initial permutation $\pi \in \Pi(\mb{v})$, the DAG induced by the output of unbounded GRaSP$_0$ is equivalent to the DAG returned by TSP.}\\

\begin{proof}
Immediate from \textbf{Lemma \ref{ct_2_Chickering}} and \textbf{Lemma \ref{Chickering_2_ct}}.
\end{proof}

Now we turn to the discussion on the correctness of GRaSP$_0$ under faithfulness.

\begin{lemma}
\label{ct_lemma}
Given a graphoid $\Prob$ over $\mb{V}$ and any $\pi \in \Pi(\mb{v})$, if $\G_\pi \notin \Pm(\Prob)$, then there exists a ct-sequence $\mf{T} = \la \pi,..., \tau\ra$ s.t. $\I(\G_\pi) \subset \I(\G_\tau)$.
\end{lemma}
\begin{proof}
Suppose that $\G_\pi \notin \Pm(\Prob)$. By \textbf{Definition \ref{P-minimal}}, it follows that there exists $\mc{H} \in \CMC(\Prob)$ s.t. $\I(\G_\pi) \subset \I(\mc{H}) \subseteq \I(\Prob)$. By \textbf{Theorem \ref{Chickering_seq}}, we know that there exists a Chickering sequence $\mf{C}_0 = \la \G_\pi = \G^1, ..., \G^l = \mc{H}\ra$. Without loss of generality, suppose that $\mf{C}_0$ is the shortest Chickering sequence where each $\G^{i+1}$ differs from $\G^i \in \SGS(\Prob)$ by the reversal of a covered edge in $\E^0(\G^i)$ for each $1 \leq i < l-1$, and $\G^l$ is obtained from deleting a directed edge in $\G^{l-1}$. Notice that $|\E(\G^l)| < |\E(\G^1)|$ due to the edge deletion. If $\G^l \in \SGS(\Prob)$, then $\mf{C}_0$ is a TSP-Chickering sequence. Otherwise, we can easily construct a TSP-Chickering sequence $\mf{C} = \la \G^1,..., \G^m\ra$ with $l-1$ as the turning index and $\G^m \in \SGS(\Prob)$ obtained by repeated edge-deletion from $\G^l$ such that $\I(\G^1) \subset \I(\G^l) \subset \I(\G^m)$. By \textbf{Lemma \ref{Chickering_2_ct}} (b), we have the desired ct-sequence.
\end{proof}

\,\\
\textbf{Theorem \ref{ct-theorem}} \hspace{0.1cm} \textit{Given a graphoid $\Prob$ over $\mb{V}$ and any $\pi \in \Pi(\mb{v})$, if $\G_\pi \notin \Pm(\Prob)$, then there exists a ct-sequence $\mf{T} = \la \pi,..., \tau\ra$ s.t. $\G_\tau \in \Pm(\Prob)$.}\\

\begin{proof}
Immediate from \textbf{Lemma \ref{ct_lemma}}.
\end{proof}

\begin{theorem}
\label{GRaSP0_correct_consistent}
Unbounded GRaSP$_0$ is correct and pointwise consistent under faithfulness.
\end{theorem}
\begin{proof}
We review the argument for the correctness of unbounded GRaSP$_0$ under faithfulness given in the main paper. Given a graphoid $\Prob$ over $\mb{V}$, consider any initial permutation $\pi \in \Pi(\mb{v})$. Given that unbounded GRaSP$_0$ greedily search for a ct-sequence from $\pi$, it is guaranteed by \textbf{Theorem \ref{ct-theorem}} that $\tau$ returned by unbounded GRaSP$_0$ in \textbf{Algorithm \ref{alg:grasp}} induces a P-minimal DAG. Under faithfulness, we have $\G_\tau \in \MEC(\G^*)$ due to \textbf{Theorem \ref{razors}} and hence unbounded GRaSP$_0$ is correct.

Alternatively, the correctness and pointwise consistency of unbounded GRaSP$_0$ can also be proven directly from \textbf{Theorem \ref{TSP=GRaSP0}} and the corresponding results of TSP in \citep{solus2021consistency}.
\end{proof}

\,\\
\textbf{Corollary \ref{correct_consistent}} \hspace{0.1cm} \textit{Unbounded GRaSP$_0$, GRaSP$_1$, and GRaSP$_2$ are correct and pointwise consistent under faithfulness.}\\

In the following, we want to prove that faithfulness is not only sufficient, but also \textit{necessary} for the correctness of TSP and unbounded GRaSP$_0$. We first want to prove an interesting and novel equivalence between two causal razors: faithfulness and u-P-minimality.

\begin{lemma}
\label{CMC_single_CI}
Given a joint probability distribution $\Prob$ over $\mb{V}$, for any $\la X_i, X_j\,|\,\mb{X}_{\mb{k}} \ra \in \I(\Prob)$, there exists $\G \in \mt{DAG}(\mb{V})$ s.t. $\I(\G) = \{\la X_i, X_j\,|\,\mb{X}_{\mb{k}} \ra\}$. 
\end{lemma}
\begin{proof}
Consider $\mb{V} = \{X_1,...,X_m\}$. An empty DAG suffices when $m = 2$. So assume that $m \geq 3$. Without loss of generality, consider $\la X_1, X_{k+2} \,|\,\mb{X}_{\mb{k}}\ra \in \I(\Prob)$ where $\mb{k} = \la 2,...,k+1\ra$, and the remaining vertices are $\la k+3,..., m \ra$. We propose a procedure which guarantees the existence of the desired DAG $\G$.

\begin{algorithm}
\DontPrintSemicolon
$\G \ot \text{a complete undirected graph over } \mb{v}$\;
remove the adjacency $1$ \textemdash $\,k+2$ in $\G$ \;
\ForEach{$(j, k)$ that are adjacent in $\G$}{
    \If{$j < k$}{
        orient $j \to k$ in $\G$
    }
}
return $\G$    
\end{algorithm}

Line 3 to 5 guarantee that $\mathcal{G}$ is a DAG since all edges are directed and pointing from lower indices to higher indices such that no directed cycle can occur. Finally, $1 \perp_\G k+2\,|\,\mb{k}$ holds because all directed paths from $1$ to $k+2$ either contain a non-collider $i \in \mb{k}$ or contain a collider $i \notin \mb{k}$. Therefore, $\I(\G) = \{\la X_1, X_{k+2}\,|\,\mb{X}_\mb{k})\}$ because no other d-separation relations hold in $\mathcal{G}$.
\end{proof}

\begin{theorem}
\label{CFC-uPm}
For any joint probability distribution $\Prob$, $\CFC(\Prob) = \uPm(\Prob)$.
\end{theorem}

\begin{proof}
[$\subseteq$] Suppose that $\G \in \CFC(\Prob)$. It follows that $\G \in \Pm(\Prob)$ by \textbf{Definition \ref{P-minimal}}. For any $\G' \in \CMC(\Prob)$, if $\I(\G') \subset \I(\G)$, then $\G' \notin \Pm(\Prob)$. Hence, if $\G' \in \Pm(\Prob)$, then $\I(\G') = \I(\G)$. Hence, $\G \in \uPm(\Prob)$.

$[\supseteq$] Suppose that $\G \notin \CFC(\Prob)$. Since $\uPm(\Prob) \subseteq \Pm(\Prob)$ by \textbf{Definition \ref{P-minimal}}, if $\G \notin \Pm(\Prob)$, we have $\G \notin \uPm(\Prob)$ immediately. So consider the case where $\G \in \Pm(\Prob)$. It follows from $\G \notin \CFC(\Prob)$ that there exists a CI relation $\psi \in \I(\Prob) \setminus \I(\G)$. By \textbf{Lemma \ref{CMC_single_CI}}, we can construct a DAG $\G^0$ such that $\I(\G^0) = \{\psi\}$. Consequently, there exists $\G^1 \in \Pm(\Prob)$ such that $\I(\G^0) \subseteq \I(\G^1) \subseteq \I(\Prob)$. Since $\psi \in \I(\G^1)$, we know that $\G^1 \notin \MEC(\G)$. Given that both $\G, \G^1 \in \Pm(\Prob)$, we have $\G \notin \uPm(\Prob)$.
\end{proof}

\,\\
\textbf{Theorem \ref{TSP_necessary}} \hspace{0.1cm} \textit{Given a graphoid $\Prob$, faithfulness is necessary for the correctness of TSP.}\\

\begin{proof}
Suppose that $(\G^*, \Prob)$ is unfaithful. We consider the two kinds of unfaithfulness in \citep{zhang2008detection}: \textit{detectable} (i.e., $\CFC(\Prob) = \varnothing$) versus \textit{undetectable} (i.e., $\G' \in \CFC(\Prob)$ where $\G' \notin \MEC(\G^*)$). For the latter, TSP can identify $\G_\tau \in \Pm(\Prob) = \CFC(\Prob) = \MEC(\G')$. However, TSP is incorrect because $\G_\tau \notin \MEC(\G^*)$. 

On the other hand, consider the case that $\CFC(\Prob) = \varnothing$. By \textbf{Theorem \ref{CFC-uPm}}, there exists $\G \in \Pm(\Prob)$ such that $\G \notin \MEC(\G^*)$ even if $\G^* \in \Pm(\Prob)$. Recall that Chickering algorithm can only allow us to traverse to a DAG $\mc{H}$ from $\G$ satisfying $\I(\G) \subseteq \I(\mc{H})$. It entails that Chickering algorithm can only obtain DAGs that are in $\MEC(\G)$ since $\G \in \Pm(\Prob)$ and hence never be able to reach $\G^*$ where $\I(\G_\pi) \nsubseteq \I(\G^*)$. Therefore, by setting $\pi$ as the initial permutation to TSP where $\G_\pi = \G$, TSP will return $\G_\pi$ incorrectly.
\end{proof}

Notice that \textbf{Theorem \ref{TSP_necessary}} is contrary to what \cite{solus2021consistency} suggested. They proposed an example arguing that TSP can be correct even under (detectable) unfaithfulness.\footnote{See Figure 2 in the \href{https://academic.oup.com/biomet/article-abstract/108/4/795/6062392?redirectedFrom=fulltext\#supplementary-data}{\textcolor{blue}{supplementary materials}} of \citep{solus2021consistency}.} However, the distribution used in the example is not a semigraphoid. This renders their example illegitimate because every joint probability distribution is a semigraphoid. 


\section{ESP and GRaSP-1}
\label{app:ESP_GRaSP1}
As shown in \textbf{Theorem \ref{TSP_necessary}} in the last section, TSP cannot be correct under unfaithfulness by choosing an arbitrary initial permutation. Consequently, one important question is how to relax the search space of TSP to identify a sparser permutation under unfaithfulness. \cite{solus2021consistency} proposed the \textit{Edge SP} (ESP) algorithm based on an assumption strictly weaker than that assumed by TSP. However, unlike TSP, they did not provide an operational version of ESP in their work. In this section, we are going to show a theorem similar to \textbf{Theorem \ref{TSP=GRaSP0}} but with respect to ESP and unbounded GRaSP$_1$. In other words, unbounded GRaSP$_1$ is an operational version of ESP. In the following, we first examine some technical notations used in \citep{mohammadi2018generalized} and \citep{solus2021consistency}. Readers are strongly suggested to visit \citep{solus2021consistency} for the full discussion of ESP and relevant notations. 

Given a set of measured variables $\mb{V}$, a \textit{permutohedron} on $\mb{v}$, denoted $\A_\mb{v}$, is the convex hull in $\mathbb{R}^{|\mb{v}|}$ of all permutations in $\Pi(\mb{v})$. In simpler terms, $\A_{\mb{v}}$ is the \textit{state space} with each \textit{state} being a permutation $\pi \in \Pi(\mb{v})$. The neighborhood of states in $\A_{\mb{v}}$ is defined by adjacent transpositions (ATs) as in \textbf{Definition \ref{AT}} (b). 

Notice that different states in $\A_\mb{v}$ can induce the same DAG given a graphoid $\Prob$. Thus, a natural way to narrow down the search space is to identify permutations inducing the same DAG. \textbf{Lemma \ref{AT_iff}} provides such a characterization. Construct $\A_{\mb{v}}(\Prob)$ by \textit{contracting} neighborhood in $\A_{\mb{v}}$ to ATs that correspond to the CI relations in $\I(\Prob)$ specified in \textbf{Lemma \ref{AT_iff}}. To be more specified, the contracted permutohedron $\A_{\mb{v}}(\Prob)$, also known as the \textit{DAG associahedron}, is the state space with each state being an induced DAG.\footnote{One can equivalently express each state in the DAG associahedron as the set of permutations which induce the same DAG. This is the original representation in \citep{mohammadi2018generalized}. However, we prefer the representation given in \citep{solus2021consistency} in the sense that one can easily compare DAGs that are in neighborhood.} Two states $\G_1, \G_2$ in $\A_{\mb{v}}(\Prob)$ are neighbors if and only if there exist $\pi^1, \pi^2 \in \Pi(\mb{v})$ s.t. $\G_{\pi^1} = \G^1$, $\G_{\pi^2} = \G^2$, and $\pi^1$ and $\pi^2$ are neighbors in the permutohedron $\A_{\mb{v}}$. As shown by \cite{mohammadi2018generalized}, the DAG associahedron is a convex polytope where each vertex of $\A_{\mb{v}}(\Prob)$ corresponds to a different DAG. 

To draw a clearer picture, consider any $\pi, \tau \in \Pi(\mb{v})$ where $\tau$ is $(j, k)$-different from $\pi$ for some $j, k \in \mb{v}$. They are neighbors in $\A_\mb{v}$ but they do not necessarily induce the same DAG. If $X_j \CI_\Prob X_k\,|\,\mb{X}_{\Pre(j, \pi)}$ holds, they induce the same DAG and thus correspond to the same state $\G_\pi$ in the DAG associahedron $\A_{\mb{v}}(\Prob)$. But if the CI relation does not hold, then $\G_\pi$ and $\G_\tau$ are neighbors in $\A_{\mb{v}}(\Prob)$. See Figure \ref{fig:permutohedron} for an example from \citep{solus2021consistency}.    

\begin{figure}[ht!]
\begin{center}
\subfloat{
\begin{tikzpicture}
\draw[fill=blue] (0,0) circle (3pt);
\draw[fill=blue] (2.5,0) circle (3pt);
\draw[fill=blue] (4,-2) circle (3pt);
\draw[fill=blue] (2.5,-4) circle (3pt);
\draw[fill=blue] (0,-4) circle (3pt);
\draw[fill=blue] (-1.5,-2) circle (3pt);
\node at (-1.3,0) {$\pi^1 = \la 1, 2, 3\ra$};
\node at (3.8,0) {$\pi^2 =\la 2, 1, 3\ra$};
\node at (2.8,-2) {$\pi^3 =\la 2, 3, 1\ra$};
\node at (3.8,-4) {$\pi^4 =\la 3, 2, 1\ra$};
\node at (-1.3,-4) {$\pi^5 =\la 3, 1, 2\ra$};
\node at (-0.2,-2) {$\pi^6 =\la 1, 3, 2\ra$};
\draw[thick] (0,0) -- (2.5,0) -- (4,-2) -- (2.5,-4) -- (0,-4) -- (-1.5,-2) -- (0,0);
\node at (0,-4.6) {};
\end{tikzpicture}}
\hspace{1.5cm}
\subfloat{
\begin{tikzpicture}[scale=0.8]
\draw[fill=red] (0,0) circle (3pt);
\draw[fill=red] (2.5,0) circle (3pt);
\draw[fill=red] (4,-2) circle (3pt);
\draw[fill=red] (2.5,-4) circle (3pt);
\draw[fill=red] (0,-4) circle (3pt);
\node at (-1,-0.5) {$\G_{\pi^1}$};
\node at (-1,0.3) {
\scalebox{0.6}{\begin{tikzpicture}
\node(X1) at (0,0) {$1$};
\node(X2) at (1,-1) {$2$};
\node(X3) at (2,0) {$3$};
\path [->] (X1) edge (X3);
\path [->] (X1) edge (X2);
\path [->] (X2) edge (X3);
\end{tikzpicture}}};
\node at (4.3,-0.3) {$\G_{\pi^2}$};
\node at (3.4,0.3) {
\scalebox{0.6}{\begin{tikzpicture}
\node(X1) at (0,0) {$1$};
\node(X2) at (1,-1) {$2$};
\node(X3) at (2,0) {$3$};
\path [->] (X1) edge (X3);
\path [->] (X2) edge (X1);
\path [->] (X2) edge (X3);
\end{tikzpicture}}};
\node at (5.9,-2.6) {$\G_{\pi^3}$};
\node at (5,-2){
\scalebox{0.6}{\begin{tikzpicture}
\node(X1) at (0,0) {$1$};
\node(X2) at (1,-1) {$2$};
\node(X3) at (2,0) {$3$};
\path [->] (X3) edge (X1);
\path [->] (X2) edge (X1);
\path [->] (X2) edge (X3);
\end{tikzpicture}}};
\node at (4.7,-5.1) {$\G_{\pi^4}$};
\node at (3.8,-4.5) {
\scalebox{0.6}{\begin{tikzpicture}
\node(X1) at (0,0) {$1$};
\node(X2) at (1,-1) {$2$};
\node(X3) at (2,0) {$3$};
\path [->] (X3) edge (X2);
\path [->] (X3) edge (X1);
\path [->] (X2) edge (X1);
\end{tikzpicture}}};
\node at (0.2,-5.1) {$\G_{\pi^5} = \G_{\pi^6}$};
\node at (-1.3,-4.5) {
\scalebox{0.6}{\begin{tikzpicture}
\node(X1) at (0,0) {$1$};
\node(X2) at (1,-1) {$2$};
\node(X3) at (2,0) {$3$};
\path [->] (X3) edge (X2);
\path [->] (X1) edge (X2);
\end{tikzpicture}}};
\draw[thick] (0,0) -- (2.5,0) -- (4,-2) -- (2.5,-4) -- (0,-4) -- (0,0);
\end{tikzpicture}}
\end{center}
\caption{Given $\mb{V} = \{X_1, X_2, X_3\}$, consider $\I(\Prob) = \{\la X_1, X_3\,|\,\varnothing\ra\}$. The diagram on the left is the permutohedron $\A_{\mb{v}}$ where each state is a permutation in $\Pi(\mb{v})$. The one on the right is the DAG associahedron $\A_{\mb{v}}(\Prob)$ where each state is a different DAG in $\SGS(\Prob)$. In particular, the two states $\pi^5$ and $\pi^6$ in $\A_{\mb{v}}$ are collapsed into a single state in $\A_{\mb{v}}(\Prob)$ because they induce the same DAG.}  
\label{fig:permutohedron}
\end{figure}
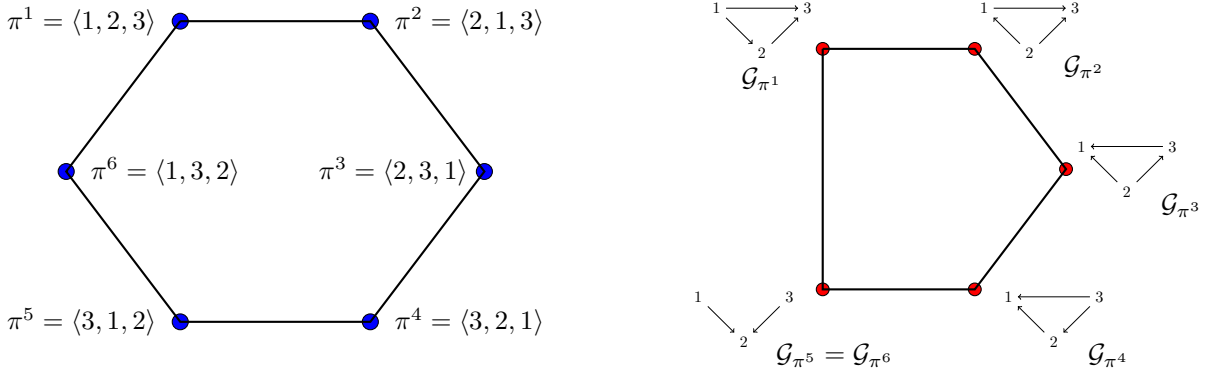

Observe that each state in the DAG associahedron $\A_{\mb{v}}(\Prob)$ corresponds to a SGS-minimal DAG according to \textbf{Theorem \ref{RU-theorem}}. ESP performs a greedy DFS in $\A_{\mb{v}}(\Prob)$. Given an initial permutation $\pi \in \Pi(\mb{v})$, set $\G$ as the induced $\G_\pi$ and traverse through $\A_{\mb{v}}(\Prob)$ by a \textit{weakly decreasing walk} to obtain $\mc{H}$ where $|\E(\mc{H})| < |\E(\G_\pi)|$.\footnote{In \citep{solus2021consistency}, their pseudocode does not indicate that such a walk needs to be weakly decreasing but such a requirement is imposed in the description of the algorithm.} If no such $\mc{H}$ exists, ESP returns $\mc{G} = \mc{G}_\pi$; else $\mc{G}$ is reset as $\mc{H}$ and repeat. 

As noted by \cite{solus2021consistency}, the construction of $\A_\mb{v}(\Prob)$ is inefficient since one is only required to know the neighboring states instead of the entire $\A_\mb{v}(\Prob)$ to perform the traversal. Below we show that unbounded GRaSP$_1$ can efficiently learn the neighbors of each state in $\A_\mb{v}(\Prob)$ by our permutation-based operation \textit{tuck} performed on singular edges. Before examining this claim, we introduce some useful definitions.

Given the permutohedron $\A_\mb{v}$, a \textit{walk} $\mf{W} = \la \pi^1,...,\pi^m\ra$ is a sequence of neighboring states in $\A_\mb{v}$ such that $\pi^i, \pi^{i+1} \in \Pi(\mb{v})$ are in AT for each $1 \leq i < m$. 

\begin{definition}
\label{walk_prop}
Given a graphoid $\Prob$ over $\mb{V}$, for any walk $\mf{W} = \la \pi^1,...,\pi^m \ra$ in $\A_\mb{v}$,
\begin{enumerate}
    \item[(a)] $\mf{W}$ is said to be \textit{DAG-preserving} if $\G_{\pi^1} = ... = \G_{\pi^m}$;
    \item[(b)] $\mf{W}$ is said to be \textit{DAG-changing} if $\la \pi^1,...,\pi^{m-1}\ra$ is DAG-preserving and $\G_{\pi^{m-1}} \neq \G_{\pi^m}$.
\end{enumerate}
\end{definition}

In addition, for each DAG-changing walk $\mf{W} = \la \pi^1,..., \pi^m\ra$, we say that $\mf{W}$ is relative to $(j,k)$ if $\pi^{m}$ is $(j, k)$-different from $\pi^{m-1}$ for some $j, k \in \mb{v}$. Thus, each DAG-changing walk is relative to a pair of vertices corresponding to the last AT performed in the walk.

\begin{definition}
\label{rev_DAGs}
Given a set of variables $\mb{V}$, consider any $j, k \in \mb{v}$. Two DAGs $\G, \mc{H} \in \DAG(\mb{V})$ are said to be $(j,k)$-reverse if $(j \to k) \in \E(\G)$ and $(k \to j) \in \E(\mc{H})$, and there does not exist any other $j', k' \in \mb{v}$ s.t. $(j' \to k') \in \E(\G)$ and $(k' \to j') \in \E(\mc{H})$. 
\end{definition}

\begin{lemma}
\label{walk_iff}
Given a graphoid $\Prob$ over $\mb{V}$, consider any $j, k \in \mb{v}$. Then there exists a DAG-changing walk $\mf{W} = \la \pi^1,...,\pi^m\ra$ relative to $(j,k)$ in $\A_\mb{v}$ if and only if $\G_{\pi^1}$ and $\G_{\pi^m}$ are neighbors in $\A_{\mb{v}}(\Prob)$ that are $(j,k)$-reverse. 
\end{lemma}
\begin{proof}
For the forward direction, given that $\pi^{m-1}$ and $\pi^m$ are $(j, k)$-different but induce different DAGs, it follows from the definition of $\A_\mb{v}(\Prob)$ that $\G_{\pi^1} = \G_{\pi^{m-1}}$ and $\G_{\pi^m}$ are neighbors in $\A_\mb{v}(\Prob)$. Also, we know that $(j \to k) \in \E(\G_{\pi^{m-1}})$ and $(k \to j) \in \E(\G_{\pi^m})$ by \textbf{Lemma \ref{AT_iff}} and (RU). The fact that $\G_{\pi^1} = \G_{\pi^{m-1}}$ and $\G_{\pi^m}$ are $(j,k)$-reverse follows immediately from (RU) and the assumption that $\pi^{m-1}$ is $(j,k)$-different from $\pi^m$. 

For the backward direction, suppose that $\G_\pi$ and $\G_\tau$ are neighbors in $\A_{\mb{v}}(\Prob)$ that are $(j, k)$-reverse. It entails from (RU) that there exist $\pi', \tau' \in \Pi(\mb{v})$ such that $\pi'$ and $\tau'$ are $(j,k)$-different where $\G_\tau = \G_{\tau'}$ and $\G_{\pi} = \G_{\pi'}$. Hence, $\la \pi', \tau'\ra$ is our desired DAG-changing walk relative to $(j, k)$ in $\A_\mb{v}$.   
\end{proof}

\begin{lemma}
\label{4_perm_lemma}
Given a graphoid $\Prob$ over $\mb{V}$, consider any pair $\pi^1, \tau^1 \in \Pi(\mb{v})$ such that $\pi^1 = \la \bs{\delta}_1, j, k, \bs{\delta}_2\ra$ for some sub-sequences $\bs{\delta}_1$, $\bs{\delta}_2$ of $\pi^1$, and $\tau^1 = \la \bs{\zeta}_1, j, k, \bs{\zeta}_2\ra$ for some sub-sequences $\bs{\zeta}_1$, $\bs{\zeta}_2$ of $\tau^1$. Further consider $\pi^2 = \la \bs{\delta}_1, k, j, \bs{\delta}_2\ra$ and $\tau^2 = \la \bs{\zeta}_1, k, j, \bs{\zeta}_2\ra$. If $\G_{\pi^1} = \G_{\tau^1}$, then $\G_{\pi^2} = \G_{\tau^2}$. 
\end{lemma}
\begin{proof}
Notice that $\G_{\pi^2} \in \SGS(\Prob)$ by \textbf{Theorem \ref{RU-theorem}}. If we can show that $\tau^2$ is a causal order of $\G_{\pi^2}$, it follows from \textbf{Lemma \ref{SGS_same_DAG}} that $\G_{\pi^2} = \G_{\tau^2}$. To do so, it suffices to show the following. For any $i \in \mb{v} \setminus \{j, k\}$,
\begin{enumerate}
    \item[(i)] if $(i \to j) \in \E(\G_{\pi^2})$, then $i \in \bs{\zeta}_1$;
    \item[(ii)] if $(i \to k) \in \E(\G_{\pi^2})$, then $i \in \bs{\zeta}_1$;
    \item[(iii)] if $(j \to i) \in \E(\G_{\pi^2})$, then $i \in \bs{\zeta}_2$;
    \item[(iv)] if $(k \to i) \in \E(\G_{\pi^2})$, then $i \in \bs{\zeta}_2$.
\end{enumerate}
For (i), suppose that $(i \to j) \in \E(\G_{\pi^2})$. If $(i \to j) \in \E(\G_{\pi^1})$ as well, then $(i \to j) \in \E(\G_{\tau^1})$ since $\G_{\pi^1} = \G_{\tau^1}$. This entails that $i \in \bs{\zeta}_1$. On the other hand, consider the case that $(i \to j) \notin \E(\G_{\pi^1})$. Then
\begin{align}
    X_i \nCI_\Prob X_j \,&|\,\mb{X}_{\bs{\delta}_1 \setminus \{i\}} \cup \{X_k\} &\because (i \to j) \in \E(\G_{\pi^2})\\
    X_i \nCI_\Prob \{X_j, X_k\} \,&|\,\mb{X}_{\bs{\delta}_1 \setminus \{i\}} &\because (4), \textit{weak union}\\
    X_i \CI_\Prob X_j \,&|\,\mb{X}_{\bs{\delta}_1 \setminus \{i\}} &\because (i \to j) \notin \E(\G_{\pi^1})\\
    X_i \nCI_\Prob X_k \,&|\,\mb{X}_{\bs{\delta}_1 \setminus \{i\}} \cup \{X_j\} &\because (5), (6), \textit{contraction}
\end{align}
By (RU), (8) entails that $(i \to k) \in \E(\G_{\pi^1}) = \E(\G_{\tau^1})$. Since $\tau^1$ is a causal order of $\G_{\tau^1}$, we have $i \in \bs{\zeta}_1$.

For (ii), suppose that $(i \to k) \in \E(\G_{\pi^2})$. Similar to (i), the case for $(i \to k) \in \E(\G_{\pi^1})$ is simple. So consider the case where $(i \to k) \notin \E(\G_{\pi^1})$.
\begin{align}
    X_i \nCI_\Prob X_k \,&|\,\mb{X}_{\bs{\delta}_1 \setminus \{i\}} &\because (i \to k) \in \E(\G_{\pi^2})\\
    X_i \nCI_\Prob \{X_j, X_k\} \,&|\,\mb{X}_{\bs{\delta}_1 \setminus \{i\}} &\because (8), \textit{decomposition}\\
    X_i \CI_\Prob X_k \,&|\,\mb{X}_{\bs{\delta}_1 \setminus \{i\}} \cup \{X_j\} &\because (i \to k) \notin \E(\G_{\pi^1})\\
    X_i \nCI_\Prob X_j \,&|\,\mb{X}_{\bs{\delta}_1 \setminus \{i\}} &\because (9), (10), \textit{contraction}
\end{align}
By (RU), (13) entails that $(i \to j) \in \E(\G_{\pi^1}) = \E(\G_{\tau^1})$ and hence $i \in \bs{\zeta}_1$.

For (iii), suppose that $(j \to i) \in \E(\G_{\pi^2})$. Then we have $(j \to i) \in \E(\G_{\pi^1})$ by (RU) because $\Pre(i, \pi^1) = \Pre(i, \pi^2)$. Hence $(j \to i) \in \E(\G_{\tau^1})$ since $\G_{\pi^1} = \G_{\tau^1}$. Given that $\tau^1$ is a causal order of $\G_{\tau^1}$, we have $i \in \bs{\zeta}_2$. (iv) is analogous to (iii).
\end{proof}

\begin{lemma}
\label{two_walks}
Given a graphoid $\Prob$ over $\mb{V}$, consider any two DAG-changing walks $\mf{W} = \la \pi^1,..., \pi^m\ra$ and $\mf{W}' = \la \tau^1,..., \tau^n\ra$ in $\A_\mb{v}$ where $\pi^1 = \tau^1$. If $\mf{W}$ and $\mf{W}'$ are both relative to the same $(j, k)$ for some $j, k \in \mb{v}$, then $\G_{\pi^m} = \G_{\tau^n}$.
\end{lemma}
\begin{proof}
Immediate from \textbf{Definition \ref{walk_prop}} and \textbf{Lemma \ref{4_perm_lemma}}.
\end{proof}

\begin{lemma}
\label{singular_lemma}
Given a graphoid $\Prob$ over $\mb{V}$, consider any DAG-changing walk $\mf{W} = \la \pi^1,..., \pi^m\ra$ in $\A_\mb{v}$ which is relative to $(j, k)$ for some $j, k \in \mb{v}$. Then $j \to k$ is a singular edge in $\G_{\pi^1}$. 
\end{lemma}
\begin{proof}
Let $\mf{W}_0$ denotes the DAG-preserving walk $\la \pi^1,..., \pi^{m-1}\ra$. Given that $\pi^{m}$ is $(j, k)$-different from $\pi^{m-1}$, it follows from \textbf{Lemma \ref{AT_iff}} and (RU) that $(j \to k) \in \E(\G_{\pi^{m-1}})$. Since $\mf{W}_0$ is a DAG-preserving walk in $\A_\mb{v}$, we have $(j \to k) \in \E(\G_{\pi^{1}}) = \E(\G_{\pi^{m-1}})$. 

Next, suppose by reductio that $j \to k$ is not a singular edge in $\G_{\pi^1}$. Then there is a unidirectional path from $j$ to $k$ other than $j \to k$ in $\G_{\pi^1}$. So there exists $l \in \mb{v}$ such that $l \in \De(j, \G_{\pi^1}) \cap \An(k, \G_{\pi^1})$. In order to ensure that $j$ and $k$ are $\pi^{m-1}$-adjacent, either $\pi^{m-1}[l] < \pi^{m-1}[j]$ or $\pi^{m-1}[l] > \pi^{m-1}[k]$ holds. However, either case will violate that $\pi^{m-1}$ is a causal order of $\G_{\pi^{m-1}} = \G_{\pi^1}$.
\end{proof}

\begin{lemma}
\label{singular_tuck_lemma}
Given a graphoid $\Prob$ over $\mb{V}$, consider $\pi \in \Pi(\mb{v})$ where $(j \to k) \in \E(\G_\pi)$ is a singular edge for some $j, k \in \mb{v}$. Then there exists a DAG-changing walk $\mf{W} = \la \pi,...,\tau\ra$ in $\A_\mb{v}$ relative to $(j, k)$ where $\tau = \textit{tuck}(\pi, j, k)$.
\end{lemma}
\begin{proof}
First, we rewrite $\pi = \la \bs{\delta}_{<j}, j, \bs{\delta}_{j \sim k}, k, \bs{\delta}_{>k}\ra$ as usual. Then we partition $\bs{\delta}_{j \sim k}$ as follows: $\bs{\zeta}_a = \la i \in \bs{\delta}_{j \sim k}: i \in \An(k, \G_{\pi})\ra$, and $\bs{\zeta}_b = \la i \in \bs{\delta}_{j \sim k}: i \notin \An(k, \G_{\pi})\ra$. Given that $(j \to k)$ is a singular edge, we know that $\De(j, \G_{\pi^1}) \cap \An(k, \G_{\pi^1}) = \varnothing$. In other words, we know that (i) each vertex in $\bs{\zeta}_a$ has no ancestor in $\bs{\delta}_{j \sim k} \setminus \bs{\zeta}_a$ in $\G_{\pi}$ and (ii) each vertex in $\bs{\zeta}_b$ has no descendant in $\bs{\delta}_{j \sim k} \setminus \bs{\zeta}_b$ in $\G_{\pi}$. 

Now consider the permutation $\tau' = \la \bs{\delta}_{<j}, \bs{\zeta}_a, j, k, \bs{\zeta}_{b}, \bs{\delta}_{>k}\ra$ in particular. We want to show that there exists a DAG-preserving walk from $\pi$ to $\tau'$. Such a walk is easy to construct. First, perform repeated ATs by moving each $i \in \bs{\zeta}_a$ prior to $j$ from left to right, and then repeated ATs by moving each $i \in \bs{\zeta}_{b}$ behind $k$ from right to left. The two sets of ATs are licensed by (i) and (ii) respectively. Hence, we have $\G_{\tau'} = \G_{\pi}$. Finally, consider $\tau = \textit{tuck}(\pi, j, k) = \la \bs{\delta}_{<j}, \bs{\zeta}_a, k, j, \bs{\zeta}_b, \bs{\delta}_{>k}\ra$ which is $(j, k)$-different from $\tau'$. By (RU) and \textbf{Lemma \ref{AT_iff}}, we know that $\G_{\tau} \neq \G_{\tau'}$ and thus $\la \pi ,..., \tau', \tau\ra$ is a DAG-changing walk in $\A_\mb{v}$ relative to $(j, k)$.
\end{proof}

\begin{theorem}
\label{singular_tuck_thm}
Given a graphoid $\Prob$ over $\mb{V}$, consider any DAG-changing walk $\mf{W} = \la \pi^1,...,\pi^m\ra$ in $\A_\mb{v}$ which is relative to $(j, k)$ for some $j, k \in \mb{v}$. Then $\G_{\pi^{m}} = \G_\tau$ where $\tau = \textit{tuck}(\pi^1, j, k)$.
\end{theorem}
\begin{proof}
We obtain a DAG-changing walk $\mf{W}' = \la \pi^1, ...,\tau\ra$ in $\A_\mb{v}$ relative to $(j, k)$ by \textbf{Lemma \ref{singular_tuck_lemma}}. Since both $\mf{W}$ and $\mf{W}'$ are relative to the same $(j, k)$, it follows from \textbf{Lemma \ref{two_walks}} that $\G_{\pi^{m}} = \G_\tau$.
\end{proof}

Similar to the discussion in Appendix \textcolor{blue}{\textbf{\ref{app:correct}}}, we want to fix the ordering of the set of singular edges in any DAG. This ensures that ESP and unbounded GRasP$_1$ will not yield different DAGs simply due to the issue of order-dependence. Below we prove that ESP and unbounded GRaSP$_1$ are equivalent algorithms.    

\begin{theorem}
\label{ESP=GRaSP1}
Given a graphoid $\Prob$ and any initial permutation $\pi \in \Pi(\mb{v})$, the DAG induced by the output of unbounded GRaSP$_1$ is equivalent to the DAG returned by ESP.
\end{theorem}
\begin{proof}
Consider any $j, k \in \mb{v}$. By \textbf{Lemma \ref{singular_lemma}} and \textbf{Lemma \ref{singular_tuck_lemma}}, every DAG-changing walk $\mf{W} = \la \pi^1,..., \pi^m\ra$ in $\A_\mb{v}$ relative to $(j, k)$ corresponds to a \textit{tuck} operation of the singular edge $j \to k$ in $\E(\G_{\pi^1})$. Hence, by \textbf{Lemma \ref{walk_iff}}, we know that $\textit{tuck}(\pi^1, j, k)$ corresponds to the neighboring relation between $\G_{\pi^1}$ and $\G_{\pi^m}$ in $\A_\mb{v}(\Prob)$ that are $(j, k)$-reverse. Therefore, every step taken by ESP to move to a neighboring state in $\A_\mb{v}(\Prob)$ (relative to a unique pair of vertices) is equivalent to the tuck operation taken by GRaSP$_1$ over the same pair of vertices.  
\end{proof}

\section{Causal Razors and GRaSP}
\label{app:tiers}
In this section, we first provide a logical analysis of the causal razors discussed in the main text.\footnote{There are other causal razors discussed in the literature, including, but not limited to, \textit{adjacency-faithfulness} and \textit{orientation-faithfulness} in \citep{ramsey2006}, and \textit{triangle-faithfulness} in \cite{zhang2013comparison}. But they do not have a strong connection with our discussion of GRaSP and so will not analyzed in this work.} Then we construct new causal razors with respect to each tier of GRaSP, and show how a higher tier of GRaSP requires a strictly weaker causal razor. 

\begin{theorem}
\label{razors_hierarchy}
The following statements are true:
\begin{enumerate}
    \item[(a)] For any joint probability distribution $\Prob$, $\uPm(\Prob) = \CFC(\Prob) \subseteq \uFr(\Prob) \subseteq \Fr(\Prob) \subseteq \Pm(\Prob) \subseteq \SGS(\Prob)$.
    \item[(b)] For any joint probability distribution $\Prob$, if faithfulness is satisfied, $\CFC(\Prob) = \uFr(\Prob) = \Fr(\Prob) = \Pm(\Prob).$
    \item[(c)] There exists a joint probability distribution s.t. $\CFC(\Prob) \subset \uFr(\Prob)$.
    \item[(d)] There exists a joint probability distribution s.t. $\uFr(\Prob) \subset \Fr(\Prob)$.
    \item[(e)] There exists a joint probability distribution s.t. $\Fr(\Prob) \subset \Pm(\Prob)$.
    \item[(f)] There exists a joint probability distribution s.t. $\Pm(\Prob) \subset \SGS(\Prob)$.
\end{enumerate}
\end{theorem}
\begin{proof}
For (a), $\uPm(\Prob) = \CFC(\Prob)$ is our result in \textbf{Theorem \ref{CFC-uPm}}. $\CFC(\Prob) \subseteq \uFr(\Prob)$ is proven in \citep{raskutti2018learning}, $\uFr(\Prob) \subseteq \Fr(\Prob)$ is true by \textbf{Definition \ref{frugal}}, $\Fr(\Prob) \subseteq \Pm(\Prob)$ in \citep{forster2020frugal}, and $\Pm(\Prob) \subseteq \SGS(\Prob)$ in \citep{zhang2013comparison}. (b) is a direct consequence of (a) and \textbf{Theorem \ref{razors}}.

For (c), see \citep[Theorem 2.4]{raskutti2018learning}. For (d), see \citep[Figure 6]{forster2020frugal}. For (e), see \citep[Theorem 2.5]{raskutti2018learning}. For (f), see \citep[Figure 2]{zhang2013comparison}. Additionally, the example in \textbf{Theorem \ref{2_not_1}} and its corresponding Figure \ref{fig:GRaSP2_better} verifies (c) and (e): $\G^* \in \uFr(\Prob) \setminus \CFC(\Prob)$ and $\G_\pi \in \Pm(\Prob) \setminus \Fr(\Prob)$. On the other hand, each of $\G_{\pi^1}, \G_{\pi^2}, \G_{\pi^3}$, and $\G_{\pi^4}$ in the DAG-associahedron in Figure \ref{fig:permutohedron} is in $\SGS(\Prob)\setminus \Pm(\Prob)$ verifying (f).
\end{proof}

\begin{definition}
(TSP-razor and ESP-razor) Given a graphoid $\Prob$ over $\mb{V}$, let $\textit{tsp}(\Prob, \pi)$ be the DAG returned by TSP on $\Prob$ by setting $\pi$ as the initial permutation. Define $\mt{TSP}(\Prob) = \{\G \in \DAG(\mb{V}): \pi \in \Pi(\mb{v})$ and $\G = tsp(\Prob, \pi)\}$ as the set of DAGs returned by TSP on $\Prob$ over each initial permutation in $\Pi(\mb{v})$. Further define
\begin{align*}
    \mt{TSPr}(\Prob) = \{\G \in \mt{TSP}(\Prob): \neg \exists \G' \in \mt{TSP}(\Prob) \text{ s.t. } \G' \notin \MEC(\G)\}.
\end{align*}
$(\G^*, \Prob)$ satisfies the TSP-razor if $\G^* \in \mt{TSPr}(\Prob)$. Similarly for ESP, $\textit{esp}$, $\mt{ESP}$, $\mt{ESPr}$, and ESP-razor. 
\end{definition}

One can observe that $\mt{TSPr}(\Prob) = \mt{TSP}(\Prob)$ if every DAG in $\mt{TSP}(\Prob)$ belongs to the same MEC, and $\mt{TSPr}(\Prob) = \varnothing$ otherwise. The same is also true for $\mt{ESPr}(\Prob)$ and $\mt{ESP}(\Prob)$. These definitions will be proven useful when we compare them with the classes of DAGs discussed in \textbf{Theorem \ref{razors_hierarchy}}. Below we provide a similar definition for each tier of GRaSP. 

\begin{definition}
(GRaSP$_t$-razor) Given a graphoid $\Prob$ over $\mb{V}$, for $t \in \{0, 1, 2\}$, define $\mt{GRaSP}_t(\Prob) = \{\G_\tau \in \DAG(\mb{V}): \pi \in \Pi(\mb{v}) \text{ and } \tau = \textit{grasp}(\Prob, \pi, |\mb{v}|, t)\}$ as the set of DAGs returned by unbounded GRaSP$_t$ on $\Prob$ over each initial permutation in $\Pi(\mb{v})$. Further define
\begin{align*}
    \mt{GRaSP}_t\mt{r}(\Prob) = \{\G \in \mt{GRaSP}_t(\Prob): \neg \exists \G' \in \mt{GRaSP}_t(\Prob) \text{ s.t. } \G' \notin \MEC(\G)\}.
\end{align*}
$(\G^*, \Prob)$ satisfies the GRaSP$_t$-razor if $\G^* \in \mt{GRaSP}_t\mt{r}(\Prob)$.
\end{definition}

\begin{theorem}
Given a graphoid $\Prob$, the following statement is true:
\begin{align*}
    \CFC(\Prob) = \mt{TSPr}(\Prob) = \mt{GRaSP}_0\mt{r}(\Prob) \subseteq \mt{ESPr}(\Prob) = \mt{GRaSP}_1\mt{r}(\Prob) \subseteq \mt{GRaSP}_2\mt{r}(\Prob) \subseteq \uFr(\Prob).
\end{align*}
\end{theorem}
\begin{proof}
$\CFC(\Prob) = \mt{TSPr}(\Prob) = \mt{GRaSP}_0\mt{r}(\Prob)$ is directly entailed by \textbf{Theorem \ref{TSP=GRaSP0}} and \textbf{Theorem \ref{TSP_necessary}}. \cite{solus2021consistency} showed that $\mt{TSPr}(\Prob) \subseteq \mt{ESPr}(\Prob)$. $\mt{ESPr}(\Prob) = \mt{GRaSP}_1\mt{r}(\Prob)$ is entailed by \textbf{Theorem \ref{ESP=GRaSP1}}. 

Next, to show that $\mt{GRaSP}_1\mt{r}(\Prob) \subseteq \mt{GRaSP}_2\mt{r}(\Prob)$, notice that $\mt{GRaSP}_1\mt{r}(\Prob) = \mt{GRaSP}_1(\Prob)$ when all DAGs in $\mt{GRaSP}_1(\Prob)$ belong to the same MEC, and $\mt{GRaSP}_1\mt{r}(\Prob) = \varnothing$ otherwise. The latter case validates $\mt{GRaSP}_1\mt{r}(\Prob) \subseteq \mt{GRaSP}_2\mt{r}(\Prob)$ trivially. Now consider the former case where all DAGs in the non-empty $\mt{GRaSP}_1(\Prob)$ belong to the same MEC and so they have the same number of edges. Now consider any $\pi \in \Pi(\mb{v})$ satisfying $\G_\pi \in \Fr(\Prob)$ (where $\Fr(\Prob)$ is necessarily non-empty). We know that $\G_\pi \in \mt{GRaSP}_1(\Prob)$. This is because every initial permutation in $\Pi(\mb{v})$ is considered and unbounded GRaSP$_1$ will never return a denser permutation than its initial permutation. Hence, every DAG in $\mt{GRaSP}_1(\Prob)$ is the sparsest Markovian DAG. (The same also holds when $\mt{GRaSP}_2\mt{r}(\Prob) \neq \varnothing$.) Then the construction of \textbf{Algorithm \ref{alg:grasp}} entails that GRaSP$_2$ will return the same permutation as GRaSP$_1$. Hence, $\mt{GRaSP}_1\mt{r}(\Prob) = \mt{GRaSP}_2\mt{r}(\Prob)$ when all DAGs in $\mt{GRaSP}_1(\Prob)$ belongs to the same MEC.

Lastly, to show that $\mt{GRaSP}_2\mt{r}(\Prob) \subseteq \uFr(\Prob)$, we use a proof similar to the above. First, the case where $\mt{GRaSP}_2\mt{r}(\Prob) = \varnothing$ is trivial. Consider the case where $\mt{GRaSP}_2\mt{r}(\Prob) = \mt{GRaSP}_2(\Prob)$ s.t. all DAGs in $\mt{GRaSP}_2(\Prob)$ are in the same MEC. Using a similar inference used in the last paragraph, we know that every DAG in $\mt{GRaSP}_2(\Prob)$ is the sparsest Markovian DAG. Therefore, $\mt{GRaSP}_2\mt{r}(\Prob) = \uFr(\Prob)$ when all DAGs in $\mt{GRaSP}_2(\Prob)$ belongs to the same MEC.
\end{proof}

\begin{theorem}
\label{ESP_not_TSP}
There exists a graphoid $\Prob$ s.t. $\mt{GRaSP}_0\mt{r}(\Prob) \subset \mt{GRaSP}_1\mt{r}(\Prob)$.
\end{theorem}
\begin{proof}
Given the equivalence between TSP and unbounded GRaSP$_0$ shown in \textbf{Theorem \ref{TSP=GRaSP0}}, and that between ESP and unbounded GRaSP$_1$ in \textbf{Theorem \ref{ESP=GRaSP1}}, we can borrow the example from \citep{solus2021consistency} on how ESP requires a strictly weaker causal razor than TSP. We refer the readers to Figure 3 in the \href{https://academic.oup.com/biomet/article-abstract/108/4/795/6062392?redirectedFrom=fulltext#supplementary-data}{\textcolor{blue}{supplementary materials}} of \citep{solus2021consistency}.
\end{proof}

In the remainder of this section, we discuss two examples: how unbounded GRaSP$_2$ requires a strictly weaker causal razor than unbounded GRaSP$_1$, and how unbounded GRaSP$_2$ requires a strictly stronger causal razor than u-frugality. The joint distribution of each example below is a compositional graphoid. For the sake of simplicity, we only include CI relations that hold between two \textit{singleton} sets of variables such that all other CI relations entailed by each of the graphoid axioms discussed in Appendix \textcolor{blue}{\textbf{\ref{app:graphoid}}} are understood.

\begin{theorem}
\label{2_not_1}
There exists a graphoid $\Prob$ s.t. $\mt{GRaSP}_1\mt{r}(\Prob) \subset \mt{GRaSP}_2\mt{r}(\Prob)$.
\end{theorem}
\begin{proof}
Given $\mb{V} = \{X_1,...,X_4\}$, consider the unfaithful model $(\G^*, \Prob)$ where the true DAG $\G^*$ is shown on the left in Figure \ref{fig:GRaSP2_better}, and $\I(\Prob) = \Phi \cup \Psi$ where $\Phi$ is the set of faithful CI relations and $\Psi$ is the set of unfaithful CI relations as listed below:
\begin{align*}
    \Phi = \{&\phi_1 = \la X_1, X_3\,|\,\{X_2\}\ra, \phi_2 = \la X_2, X_4\,|\,\{X_1, X_3\}\ra\};\\
    \Psi = \{&\psi_1 = \la X_2, X_4\,|\,\varnothing\ra\}.
\end{align*}
For every $\G \in \CMC(\Prob)$ where $\psi_1 \in \I(\G)$, we have $5 = |\E(\G)| > |\E(\G^*)| = 4$. Also, all 4-edge Markovian DAGs are in the same MEC. Hence, u-frugality is satisfied. Consider feeding the initial permutation $\pi = \la 2, 4, 1, 3\ra$ to unbounded GRaSP$_1$. It will return the same $\pi$ after the DFS procedure and the induced $\G_\pi$, as shown on the right in Figure \ref{fig:GRaSP2_better}, contains 5 edges. Therefore, unbounded GRaSP$_1$ fails to return the sparsest permutation under some initial permutation and $\mt{GRaSP}_1\mt{r}(\Prob) = \varnothing$. 

On the contrary, $|\mb{v}|! = 24$ initial permutations have been tested on unbounded GRaSP$_2$ and it returns $\hat{\tau}$ where $\G_{\hat{\tau}} \in \MEC(\G^*)$ for each initial permutation. Hence, $\mt{GRaSP}_2\mt{r}(\Prob) \neq \varnothing$.
\end{proof}

\begin{center}
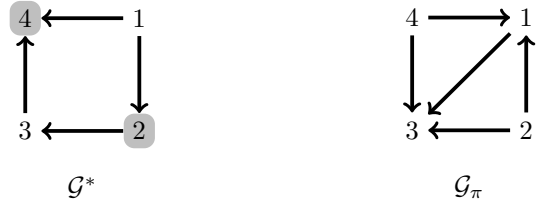
\begin{figure}[H]
    \centering
    \subfloat{
\begin{tikzpicture}[scale=1.5, roundnode/.style={circle, draw=black!60, very thick, minimum size=5mm}]
\node (X1) at (0,1) {$1$};
\node [fill=lightgray, rounded corners](X2) at (0,0) {$2$};
\node (X3) at (-1,0) {$3$};
\node [fill=lightgray, rounded corners](X4) at (-1,1) {$4$};
\node (label) at (-0.5, -0.5) {$\G^*$};
\path [->,line width=0.5mm] (X1) edge (X2);
\path [->,line width=0.5mm] (X2) edge (X3);
\path [->,line width=0.5mm] (X3) edge (X4);
\path [->,line width=0.5mm] (X1) edge (X4);
\end{tikzpicture}}
\hspace{3cm}
\subfloat{
\begin{tikzpicture}[scale=1.5, roundnode/.style={circle, draw=black!60, very thick, minimum size=5mm}]
\node (X1) at (0,1) {$1$};
\node (X2) at (0,0) {$2$};
\node (X3) at (-1,0) {$3$};
\node (X4) at (-1,1) {$4$};
\node (label) at (-0.5, -0.5) {$\G_\pi$};
\path [->,line width=0.5mm] (X4) edge (X1);
\path [->,line width=0.5mm] (X4) edge (X3);
\path [->,line width=0.5mm] (X2) edge (X1);
\path [->,line width=0.5mm] (X2) edge (X3);
\path [->,line width=0.5mm] (X1) edge (X3);
\end{tikzpicture}}
    \caption{An unfaithful model satisfying u-frugality. The true DAG $\G^*$ is shown on the left where the two shaded vertices indicate the unfaithful marginal independence $X_2 \CI_\Prob X_4\,|\,\varnothing$. Unbounded GRaSP$_1$ returns its initial permutation $\pi = \la 2, 4, 1, 3\ra$. The induced DAG $\G_\pi$ is shown on the right with 5 edges. However, unbounded GRaSP$_2$ manages to return one of the sparsest permutations under every initial permutation.}
    \label{fig:GRaSP2_better}
\end{figure}
\end{center}


\begin{theorem}
\label{uFr_not_2}
There exists a graphoid $\Prob$ s.t. $\mt{GRaSP}_2\mt{r}(\Prob) \subset \uFr(\Prob)$.
\end{theorem}
\begin{proof}
The example below is one of the uDAGs studied in Table \ref{tab:unit_test} in Section \textcolor{blue}{\textbf{\ref{u_fru_unfaithful}}} where GRaSP$_2$ fails to return one of the sparsest permutations under u-frugality. Given $\mb{V} = \{X_1,...,X_5\}$, consider the unfaithful model $(\G^*, \Prob)$ where the true DAG $\G^*$ is shown on the left in Figure \ref{fig:uFruNotGRaSP2}, and $\I(\Prob) = \Phi \cup \Psi$ where $\Phi$ is the set of faithful CI relations and $\Psi$ is the set of unfaithful CI relations as listed below:
\begin{align*}
    \Phi = \{&\phi_1 = \la X_1, X_2\,|\,\varnothing\ra, \phi_2 = \la X_1, X_2\,|\,\{X_3\}\ra,\\
    &\phi_3 = \la X_2, X_3\,|\,\varnothing\ra, \phi_4 = \la X_2, X_3\,|\,\{X_1\}\ra,\\
    &\phi_5= \la X_2, X_5\,|\,\{X_1, X_3, X_4\}\ra\};\\
    \Psi = \{&\psi_1 = \la X_1, X_5\,|\,\varnothing\ra\}.
\end{align*}
For every $\G \in \CMC(\Prob)$ where $\psi_1 \in \I(\G)$, we have $|\E(\G)| > |\E(\G^*)| = 7$. Also, all 7-edge Markovian DAGs are in the same MEC and there exists no sparser Markovian DAG. Hence, u-frugality is satisfied s.t. $\uFr(\Prob) \neq \varnothing$. 

Next, consider feeding the initial permutation $\pi = \la 5, 1, 3, 4, 2\ra$ to unbounded GRaSP$_2$. It will return the same $\pi$ after the DFS procedure and the induced $\G_\pi$, as shown on the right in Figure \ref{fig:uFruNotGRaSP2}, contains 8 edges. Therefore, unbounded GRaSP$_2$ fails to return one of the sparsest permutations under some initial permutation and $\mt{GRaSP}_2\mt{r}(\Prob) = \varnothing$. 
\end{proof}


\begin{center}
\begin{figure}[H]
    \centering
    \subfloat{
\begin{tikzpicture}[scale=1.2, roundnode/.style={circle, draw=black!60, very thick, minimum size=5mm}]
\node [fill=lightgray, rounded corners](X1) at (0.0,1.0) {$1$};
\node (X2) at (0.9510565162951535,0.30901699437494745) {$2$};
\node (X3) at (0.5877852522924732,-0.8090169943749473) {$3$};
\node (X4) at (-0.587785252292473,-0.8090169943749476) {$4$};
\node [fill=lightgray, rounded corners](X5) at (-0.9510565162951536,0.30901699437494723) {$5$};
\node (label) at (0, -1.2) {$\G^*$};
\path [->,line width=0.5mm] (X1) edge (X3);
\path [->,line width=0.5mm] (X1) edge (X4);
\path [->,line width=0.5mm] (X2) edge (X4);
\path [->,line width=0.5mm] (X3) edge (X4);
\path [->,line width=0.5mm] (X1) edge (X5);
\path [->,line width=0.5mm] (X3) edge (X5);
\path [->,line width=0.5mm] (X4) edge (X5);
\end{tikzpicture}}
\hspace{3cm}
\subfloat{
\begin{tikzpicture}[scale=1.2, roundnode/.style={circle, draw=black!60, very thick, minimum size=5mm}]
\node (X1) at (0.0,1.0) {$1$};
\node (X2) at (0.9510565162951535,0.30901699437494745) {$2$};
\node (X3) at (0.5877852522924732,-0.8090169943749473) {$3$};
\node (X4) at (-0.587785252292473,-0.8090169943749476) {$4$};
\node (X5) at (-0.9510565162951536,0.30901699437494723) {$5$};
\node (label) at (0, -1.2) {$\G_\pi$};
\path [->,line width=0.5mm] (X1) edge (X2);
\path [->,line width=0.5mm] (X1) edge (X3);
\path [->,line width=0.5mm] (X1) edge (X4);
\path [->,line width=0.5mm] (X3) edge (X2);
\path [->,line width=0.5mm] (X3) edge (X4);
\path [->,line width=0.5mm] (X4) edge (X2);
\path [->,line width=0.5mm] (X5) edge (X3);
\path [->,line width=0.5mm] (X5) edge (X4);
\end{tikzpicture}}
    \caption{An unfaithful model satisfying u-frugality. The true DAG $\G^*$ is shown on the left where the two shaded vertices indicate the unfaithful marginal independence $X_1 \CI_\Prob X_5\,|\,\varnothing$. Unbounded GRaSP$_2$ returns its initial permutation $\pi = \la 5, 1, 3, 4, 2\ra$. The induced DAG $\G_\pi$ is shown on the right with 8 edges. Hence, GRaSP$_2$ is not correct under u-frugality alone.}
    \label{fig:uFruNotGRaSP2}
\end{figure}
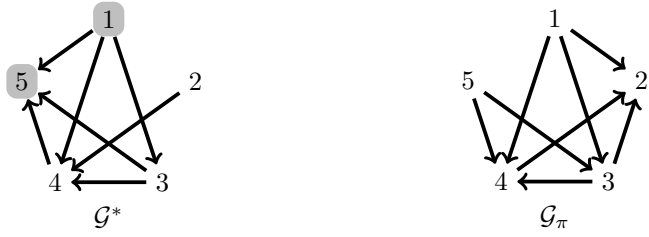
\end{center}


\textbf{Corollary \ref{coro_GRaSP_hierarchy}} \hspace{0.1cm} \textit{Given a graphoid $\Prob$, unbounded GRaSP$_2$ is correct under a strictly weaker causal razor than unbounded GRaSP$_1$, which is correct under a strictly weaker causal razor than unbounded GRaSP$_0$.}


\section{Grow-Shrink Algorithm and its properties}
\label{app:gs}
\begin{definition}
\label{BIC}
Given an observational dataset $\D$ with $n$ i.i.d. observations from a joint probability distribution $\Prob$ over $\mb{V}$ that belongs to a curved exponential family\footnote{See \citep{Kass:1437490} for an in-depth analysis of curved exponential families.}, for every $X \in \mb V$ and every $\mb M \subseteq \mb V \setminus X$,
\[
    \BIC_\D(X, \mb{M}) = \ell_{X \mid \mb{M}}(\hat{\theta}_\mt{mle} \mid \D) + c \, \frac{|\hat{\theta}_\mt{mle}|}{2} \log(n)
\]
where $\ell_{X \mid \mb M}$ is the conditional log likelihood function, $|\hat{\theta}_\mt{mle}|$ is the absolute value of the maximum likelihood estimate, and $c$ is a multiplier for the parameter penalty.
\end{definition}

BIC score is a \textit{decomposable} scoring function in the sense that the BIC score of any DAG $\G$ (over the same set of variables $\mb{V}$ as the observational dataset $\D$), denoted as $\BIC_\D(\G)$, satisfies the following:
\[
    \BIC_\D(\G) = \sum_{i \in \mb{v}} \BIC_\D(X_i, \mb{X}_{\Pa(i, \G)}).
\]

In addition, since we will be using BIC throughout this appendix, we assume that every joint probability distribution $\Prob$ belongs to a curved exponential family in this section.

\begin{algorithm}[H]
\DontPrintSemicolon
\caption{\textsc{Grow: }$\textit{grow}(\mc{D}, X, \mb{Z})$}
\label{alg:grow}
\KwIn{(a) $\D$: an observational dataset over $\mb{V}$; (b) $X \in \mb{V}$; (c) $\mb{Z} \subseteq \mb{V} \setminus \{X\}$}
\KwOut{$\mb{M}_{gr}\subseteq \mb{Z}$}
$s \ot \BIC_\D(X, \varnothing)$\;
$s' \ot s$\;
$\mb{M}_{gr} \ot \varnothing$\;
\Do{$s' > s$}{
    $s \ot s'$\;
    $s' \ot \max_{Y \in \mathbf{Z}\setminus \mathbf{M}_{gr}} \mathtt{BIC}_\mathcal{D}(X, \mathbf{M}_{gr} \cup \{Y\})$\;
    $Y' \ot \text{argmax}_{Y \in \mathbf{Z}\setminus \mathbf{M}_{gr}} \mathtt{BIC}_\mathcal{D}(X, \mathbf{M}_{gr} \cup \{Y\})$\;
    \If{$s' > s$}{
    $\mb{M}_{gr} \ot \mb{M}_{gr} \cup \{Y'\}$}
}
return $\mb{M}_{gr}$
\end{algorithm}

\begin{algorithm}[H]
\DontPrintSemicolon
\caption{\textsc{Shrink: }$\textit{shrink}(\mc{D}, X, \mb{Z})$}
\label{alg:shrink}
\KwIn{(a) $\D$: an observational dataset over $\mb{V}$; (b) $X \in \mb{V}$; (c) $\mb{Z} \subseteq \mb{V} \setminus \{X\}$}
\KwOut{(i) $\mb{M}_{sh} \subseteq \mb{Z}$; (ii) $s = \BIC_\D(X, \mb{M}_{sh})$}
$s \ot \BIC_\D(X, \mb{Z})$\;
$s' \ot s$\;
$\mb{M}_{sh} \ot \mb{Z}$\;
\Do{$s' > s$}{
    $s \ot s'$\;
    $s' \ot \max_{Y \in \mathbf{M}_{sh}} \mathtt{BIC}_\mathcal{D}(X, \mathbf{M}_{sh} \setminus \{Y\})$\;
    $Y' \ot \text{argmax}_{Y \in \mathbf{M}_{sh}} \mathtt{BIC}_\mathcal{D}(X, \mathbf{M}_{gs} \setminus \{Y\})$\;
    \If{$s' > s$}{
    $\mb{M}_{sh} \ot \mb{M}_{sh} \setminus \{Y'\}$}
}
return $\mb{M}_{sh}, s$
\end{algorithm}

\begin{theorem}
\label{BIC_local_DAG}
\citep{chickering2002optimal} Given an observational dataset $\D$ with $n$ i.i.d. observations from a joint probability distribution $\Prob$ over $\mb{V}$, consider $\G, \G' \in \DAG(\mb{V})$ where $\G'$ is resulted from adding the edge $j \to k$ in $\G$. In the large sample limit of $n$,
\begin{enumerate}
    \item[(a)] if $X_j \nCI_\Prob X_k\,|\,\mb{X}_{\Pa(k, \G)}$, then $\BIC_\D(\G') > \BIC_\D(\G)$;
    \item[(b)] if $X_j \CI_\Prob X_k\,|\,\mb{X}_{\Pa(k, \G)}$, then $\BIC_\D(\G') < \BIC_\D(\G)$.
\end{enumerate}
\end{theorem}

The theorem above is known as the \textit{local consistency} of BIC score over DAGs. We can easily derive a lemma which concerns the BIC score of a variable (relative to a set of variables). 

\begin{lemma}
\label{BIC_local_var}
Given an observational dataset $\D$ with $n$ i.i.d. observations from a joint probability distribution $\Prob$ over $\mb{V}$, consider any distinct $j, k \in \mb{v}$ and $\mb{i} \subseteq \mb{v} \setminus \{j, k\}$. In the large sample limit of $n$,
\begin{enumerate}
    \item[(a)] if $X_j \nCI_\Prob X_k\,|\,\mb{X}_\mb{i}$, then $\BIC_\D(X_k, \mb{X}_\mb{i} \cup \{X_j\}) > \BIC_\D(X_k, \mb{X}_\mb{i})$;
    \item[(b)] if $X_j \CI_\Prob X_k\,|\,\mb{X}_\mb{i}$, then $\BIC_\D(X_k, \mb{X}_\mb{i} \cup \{X_j\}) < \BIC_\D(X_k, \mb{X}_\mb{i})$.
\end{enumerate}
\end{lemma}
\begin{proof}
Construct a DAG $\G \in \DAG(\mb{V})$ by drawing all and only directed edges from each vertex in $\mb{i}$ to $k$, and another DAG $\G' \in \DAG(\mb{V})$ by adding $j \to k$ in $\G$. Then the lemma immediately follows from \textbf{Theorem \ref{BIC_local_DAG}} and the decomposable feature of BIC scores.  
\end{proof}

\begin{lemma}
\label{grow_lemma}
Consider an observational dataset $\D$ with $n$ i.i.d. observations from a compositional graphoid $\Prob$ over $\mb{V}$. In the large sample limit of $n$, for any $X \in \mb{V}$ and any $\mb{Z} \subseteq \mb{V} \setminus \{X\}$, $\MB(X, \mb{Z}) \subseteq \mb{M}_{gr}$ where $\mb{M}_{gr} = \textit{grow}(\D, X, \mb{Z}) \subseteq \mb{Z}$.
\end{lemma}
\begin{proof}
First, \textbf{Algorithm \ref{alg:grow}} requires that $\mb{M}_{gr} \subseteq \mb{Z}$, and $\BIC_\D(X, \mb{M}_{gr} \cup \{Y\}) < \BIC_D(X, \mb{M}_{gr})$ for every $Y \in \mb{Z} \setminus \mb{M}_{gr}$. By \textbf{Lemma \ref{BIC_local_var}}, we have $X \CI_\Prob Y\,|\,\mb{M}_{gr}$ for each $Y \in \mb{Z} \setminus \mb{M}_{gr}$. By composition, we have $X \CI_\Prob (\mb{Z} \setminus \mb{M}_{gr})\,|\,\mb{M}_{gr}$. Therefore, by \textbf{Definition \ref{MB}} and \textbf{Lemma \ref{MB_unique}}, we have $\MB(X, \mb{Z}) \subseteq \mb{M}_{gr}$.   
\end{proof}

\begin{lemma}
\label{shrink_lemma}
Consider an observational dataset $\D$ with $n$ i.i.d. observations from a graphoid $\Prob$ over $\mb{V}$. In the large sample limit of $n$, for any $X \in \mb{V}$ and any $\mb{Z} \subseteq \mb{V} \setminus \{X\}$, $\MB(X, \mb{Z}) = \mb{M}_{sh}$ where $\mb{M}_{sh} = \textit{shrink}(\D, X, \mb{Z}) \subseteq \mb{Z}$.
\end{lemma}
\begin{proof}
We show the lemma by $\mb{M}_{sh} \subseteq \MB(X, \mb{Z})$ and $\mb{M}_{sh} \supseteq \MB(X, \mb{Z})$.

[$\subseteq$] By reductio, suppose that there exists $Y \in \mb{M}_{sh} \subseteq \mb{Z}$ but $Y \notin \MB(X, \mb{Z})$. Let $\mb{S}$ be $\mb{M}_{sh}\setminus \{Y\}$. \textbf{Algorithm \ref{alg:shrink}} requires that $\BIC_\D(X, \mb{M}_{sh} \setminus \{Y\}) < \BIC_D(X, \mb{M}_{sh})$. In other words, we have $\BIC_D(X, \mb{S}) < \BIC_\D(X, \mb{S} \cup \{Y\})$. By \textbf{Lemma \ref{BIC_local_var}}, we have $X \nCI_\Prob Y\,|\,\mb{S}$. 

Let $\mb{W} = \mb{S} \setminus \MB(X, \mb{Z})$. From $Y \notin \MB(X, \mb{Z})$ and $Y \notin \mb{S}$, we have $\{Y\} \cup \mb{W} \subseteq \mb{Z} \setminus \MB(X, \mb{Z})$. Recall \textbf{Definition \ref{MB}} that $X \CI_\Prob \mb{Z} \setminus \MB(X, \mb{Z}) \,|\,\MB(X, \mb{Z})$. Thus, 
\begin{align}
    X \CI_\Prob&\,\{Y\} \cup \mb{W}\,|\,\MB(X, \mb{Z}) &\because X \CI_\Prob \mb{Z} \setminus \MB(X, \mb{Z}) \,|\,\MB(X, \mb{Z}),  \textit{ decomposition}\\
    X \CI_\Prob&\,Y \,|\,\MB(X, \mb{Z}) \cup \mb{W} &\because (12),  \textit{ weak union}\\
    X \CI_\Prob&\,Y\,|\,\mb{S} &\because (13), \mb{W} = \mb{S} \setminus \MB(X, \mb{Z})
\end{align}
Contradiction arises with $X \nCI_\Prob Y\,|\,\mb{S}$.

[$\supseteq$] Observe that \textbf{Algorithm \ref{alg:shrink}} removes one variable in $\mb{Z}$ one at a time repeatedly to form $\mb{M}_{sh}$. Thus, the shrink-procedure corresponds to a sequence of sets of variables $\la \mb{M}^0, ..., \mb{M}^k\ra$ and a sequence of variables $\mb{W} = \la W_1,..., W_k\ra = \mb{Z} \setminus \mb{M}_{sh}$ such that $\mb{M}^0 = \mb{Z}$, $\mb{M}^k = \mb{M}_{sh}$, and $\mb{M}^i = \mb{M}^{i-1} \setminus \{W_i\}$ (where $W_i \in \mb{M}^{i-1}$) for each $1 < i \leq k$.

Notice that $\mb{M}^{i-1} = \mb{M}^i \cup \{W_i\}$. \textbf{Algorithm \ref{alg:shrink}} requires that $\BIC_\D(X, \mb{M}^i) > \BIC_\D(X, \mb{M}^{i-1}) = \BIC_D(X, \mb{M}^i \cup \{W_i\})$. We then have 
\begin{align}
    X \CI_\Prob&\,W_1\,|\,\mb{M}^1 &\because \BIC_\D(X, \mb{M}^1) > \BIC_\D(X, \mb{M}^{0}), \textbf{Lemma \ref{BIC_local_var}}\\
    X \CI_\Prob&\,W_{2}\,|\,\mb{M}^{2} &\because \BIC_\D(X, \mb{M}^{2}) > \BIC_\D(X, \mb{M}^{1}), \textbf{Lemma \ref{BIC_local_var}}\\
    X \CI_\Prob&\,W_1\,|\,\mb{M}^{2} \cup \{W_{2}\} & \because (15), \mb{M}^1 = \mb{M}^{2} \cup \{W_{2}\}\\
    X \CI_\Prob&\,\{W_1, W_{2}\}\,|\,\mb{M}^{2} & \because (16), (17), \textit{contraction}\\
    &\vdots \nonumber\\
    X \CI_\Prob&\,\{W_1,...,W_k\}\,|\,\mb{M}^k & \because ..., \textit{contraction}\\
    X \CI_\Prob&\,\mb{W}\,|\,\mb{M}_{sh} &\because (19), \mb{W} = \la W_1,...,W_k\ra \text{ and } \mb{M}^k = \mb{M}_{sh}\\
    X \CI_\Prob&\,\mb{Z} \setminus \mb{M}_{sh}\,|\,\mb{M}_{sh} &\because (20), \mb{W} = \mb{Z} \setminus \mb{M}_{sh}
\end{align}
Hence, it follows from \textbf{Definition \ref{MB}} that $\mb{M}_{sh} \supseteq \MB(X, \mb{Z})$.
\end{proof}

\begin{theorem}
\label{gs_thm_app}
Consider an observational dataset $\D$ with $n$ i.i.d. observations from a compositional graphoid $\Prob$ over $\mb{V}$. In the large sample limit of $n$, for any $X \in \mb{V}$ and any $\mb{Z} \subseteq \mb{V} \setminus \{X\}$, $\MB(X, \mb{Z}) = \mb{M}_{gs}$ where $\mb{M}_{gs} = \textit{shrink}(\D, X, \textit{grow}(\D, X, \mb{Z}))$.
\end{theorem}
\begin{proof}
Immediate from \textbf{Lemma \ref{grow_lemma}} and \textbf{Lemma \ref{shrink_lemma}}.
\end{proof}

\begin{theorem}
Consider an observational dataset $\D$ with $n$ i.i.d. observations from a (compositional) graphoid $\Prob$ over $\mb{V} = \{X_1,...,X_m\}$, and any $\pi \in \Pi(\mb{v})$. Let $s_{i}$ and $\mb{M}_i$ be the score and the set of variables returned by \textit{shrink}$(\D, X_i, \mb{X}_{\Pre(i, \pi)})$ (or \textit{shrink}$(\D, X_i,  \textit{grow}(\D, X_i, \mb{X}_{\Pre(i, \pi)}))$ if $\Prob$ is a compositional graphoid) respectively. Denote $s_\pi$ as $\sum_{i \in \mb{v}} s_i$. In the large sample limit of $n$, $\BIC_\D(\G_\pi) = s_\pi$ where $\G_\pi$ is induced from $\pi$ by (VP).
\end{theorem}
\begin{proof}
Immediate from the decomposable feature of BIC scores, \textbf{\textbf{Lemma} \ref{shrink_lemma}} and \textbf{Theorem \ref{gs_thm_app}}. 
\end{proof}

Lastly, though a compositional graphoid is a sufficient condition for the correct identification of the unique Markov boundary using the grow-shrink algorithm, we are aware of an assumption weaker than compositional graphoid to validate such an identification. Nevertheless, this discussion will be beyond the scope of this paper and we will leave the formal proof to future work.




\section{Additional Examples}
\label{EmpiricalAppendix}

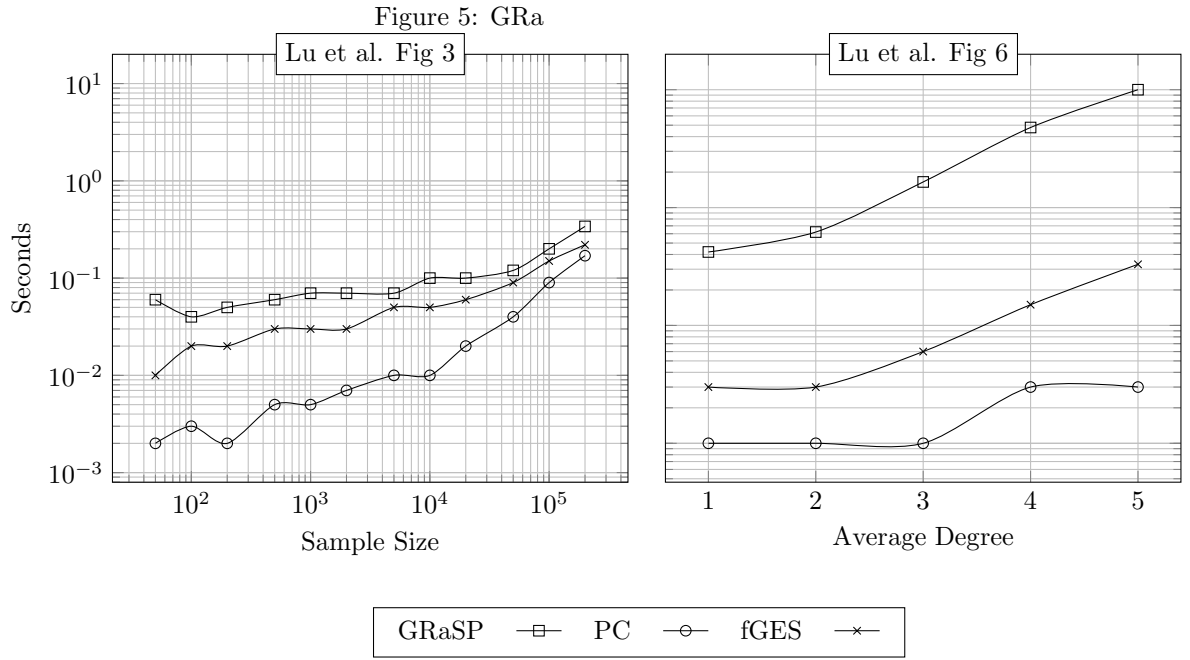
\subsection{Lu et al. Comparison}
\label{luetal-comparison}

Reported below are average statistics obtained by running GRaSP$_2$ on the published datasets used to generate Figure 6 in \citep{lu2021improving}\footnote{\url{https://github.com/ninalu/urlearning-cpp/tree/master/triplet_data}}. We cannot compare these results to Lu et al. precisely, since their statistics are given in figures and not exactly in tables, though judging from their figures it appears that GRaSP$_2$ is dominating for adjacency precision and recall, arrowhead recall, and most results for arrowhead precision. Timing results are not reported by Lu et al.; we include these to show that GRASP$_2$ returns quickly for all of these examples, where we know (personal communication) that some of the results for Triple A$^*$ take much longer. Adjacencies in these graphs are sampled with uniform probability, ``Edge-prob''.

\begin{table}[H]
    \centering
    \begin{tabular}{r|c|c|c|c|c|c}
    	Edge-prob & 0.03 & 0.04 & 0.05 & 0.06 & 0.07 & 0.08 \\
    	\hline
        Precision & 0.964 & 0.976 & 0.979 & 0.980 & 0.982 & 0.976 \\
        Recall & 0.985 & 0.982 & 0.986 & 0.986 & 0.985 & 0.985 \\
        F1 & 0.974 & 0.979 & 0.983 & 0.983 & 0.983 & 0.980
    \end{tabular}
    \caption{GRaSP$_2$ Adjacency Statistics}
    \label{tab:grasp_on_lu_adj}
\end{table}

\begin{table}[H]
    \centering
    \begin{tabular}{r|c|c|c|c|c|c}
    	Edge-prob & 0.03 & 0.04 & 0.05 & 0.06 & 0.07 & 0.08 \\
    	\hline
        Precision & 0.907 & 0.914 & 0.933 & 0.949 & 0.946 & 0.945 \\
        Recall & 0.897 & 0.916 & 0.933 & 0.952 & 0.952 & 0.955 \\
        F1 & 0.898 & 0.913 & 0.932 & 0.950 & 0.948 & 0.950
    \end{tabular}
    \caption{GRaSP$_2$ Arrowhead Statistics}
    \label{tab:grasp_on_lu_arr}
\end{table}

\begin{table}[H]
    \centering
    \begin{tabular}{r|c|c|c|c|c|c}
    	Edge-prob & 0.03 & 0.04 & 0.05 & 0.06 & 0.07 & 0.08 \\
    	\hline
        Seconds & 0.405 & 0.755 & 1.403 & 2.703 & 4.795 & 7.161
    \end{tabular}
    \caption{GRaSP$_2$ Timing Statistics}
    \label{tab:grasp_on_lu_time}
\end{table}


\subsection{Airfoil Example}
\label{airfoil-example}

Figure \ref{fig:airfoil} gives the results of running GRaSP$_2$, PC, and fGES on the Airfoil empirical example described in Section \textcolor{blue}{\textbf{\ref{sec:empirical}}}. GRaSP$_2$ gets the same uniquely frugal result as SP. To improve readability, we use the names of the variables (instead of numerals) to label the vertices. 


\begin{figure}[h!]
\begin{center}
\subfloat{
\begin{tikzpicture}[scale=1.0, roundnode/.style={circle, draw=black!60, very thick, minimum size=5mm}]
\node (l) at (3,-1) {(a) GRaSP$_2$ result};
\node (a) at (3,3) {\textit{Attack}};
\node (v) at (1.75,2) {\textit{Velocity}};
\node (c) at (4.25,2) {\textit{Chord}};
\node (p) at (1.75,1) {\textit{Pressure}};
\node (d) at (4.25,1) {\textit{Displacement}};
\node (f) at (3,0) {\textit{Frequency}};
\path [->,line width=0.5mm] (a) edge (p);
\path [->,line width=0.5mm] (a) edge (d);
\path [->,line width=0.5mm] (v) edge (a);
\path [->,line width=0.5mm] (v) edge (p);
\path [-,line width=0.5mm] (v) edge (f);
\path [->,line width=0.5mm] (c) edge (a);
\path [->,line width=0.5mm] (c) edge (d);
\path [->,line width=0.5mm] (c) edge (p);
\path [->,line width=0.5mm] (d) edge (p);
\path [->,line width=0.5mm] (f) edge (a);
\path [->,line width=0.5mm] (f) edge (p);
\end{tikzpicture}}
\hspace{1.5cm}
\subfloat{
\begin{tikzpicture}[scale=1.0, roundnode/.style={circle, draw=black!60, very thick, minimum size=5mm}]
\node (l) at (3,-1) {(b) fGES result};
\node (a) at (3,3) {\textit{Attack}};
\node (v) at (1.75,2) {\textit{Velocity}};
\node (c) at (4.25,2) {\textit{Chord}};
\node (p) at (1.75,1) {\textit{Pressure}};
\node (d) at (4.25,1) {\textit{Displacement}};
\node (f) at (3,0) {\textit{Frequency}};
\path [->,line width=0.5mm] (a) edge (p);
\path [->,line width=0.5mm] (a) edge (f);
\path [->,line width=0.5mm] (v) edge (a);
\path [->,line width=0.5mm] (v) edge (p);
\path [->,line width=0.5mm] (v) edge (f);
\path [->,line width=0.5mm] (c) edge (a);
\path [-,line width=0.5mm] (c) edge (d);
\path [->,line width=0.5mm] (c) edge (p);
\path [->,line width=0.5mm] (c) edge (f);
\path [->,line width=0.5mm] (d) edge (p);
\path [->,line width=0.5mm] (d) edge (a);
\path [->,line width=0.5mm] (f) edge (p);
\end{tikzpicture}}
\hspace{1.5cm}
\subfloat{
\begin{tikzpicture}[scale=1.0, roundnode/.style={circle, draw=black!60, very thick, minimum size=5mm}]
\node (l) at (3,-1) {(c) PC result};
\node (a) at (3,3) {\textit{Attack}};
\node (v) at (1.75,2) {\textit{Velocity}};
\node (c) at (4.25,2) {\textit{Chord}};
\node (p) at (1.75,1) {\textit{Pressure}};
\node (d) at (4.25,1) {\textit{Displacement}};
\node (f) at (3,0) {\textit{Frequency}};
\path [->,line width=0.5mm] (a) edge (c);
\path [->,line width=0.5mm] (a) edge (f);
\path [-,line width=0.5mm] (a) edge (d);
\path [->,line width=0.5mm] (v) edge (p);
\path [->,line width=0.5mm] (v) edge (f);
\path [->,line width=0.5mm] (p) edge (c);
\path [->,line width=0.5mm] (d) edge (p);
\path [->,line width=0.5mm] (d) edge (c);
\path [->,line width=0.5mm] (f) edge (p);
\end{tikzpicture}}
\end{center}
\caption{Results of algorithms on NASA airfoil experiment.}
\label{fig:airfoil}
\end{figure}

Note that both the GRaSP$_2$ and FGES results use the linear, Gaussian BIC score with a penalty multiplier of 2. For the GRaSP$_2$ result in (a), \textit{Attack} is not exogenous, which is counter-intuitive, since it is experimentally controlled. Allowing for latent variables could resolve this issues. However, we leave the development of such an algorithm to future work. On the other hand, the FGES result in (b) is notably not the same as the SP result and so is not frugal. Also, the orientation between \textit{Attack} and \textit{Displacement} is reversed.

The PC result in (c), which uses the zero partial correlation test with a significance level of 0.01, in fact has fewer edges than the frugal result and makes \textit{Chord}, another experimental variable, endogenous. Causally, PC is giving incorrect and incomplete information.


\section{Unit tests}
\label{app:unit_tests}
We consider path cancellations in DAGs between pairs of vertices, one of which is exogenous, connected by two or more unique treks. Furthermore, the path cancellations we consider elicit a marginal independence between the two vertices in question. Below, we enumerate all possible path cancellations of this type (up to vertex relabeling). Each graph illustrates a case where an unfaithful marginal independence is elicited between the two gray vertices due to path cancellation. A complete list of all unfaithful CI relations (symmetry assumed) where the independent sets are singletons is also provided for each graph.

\input{Figures/Figure_uDAGs5}








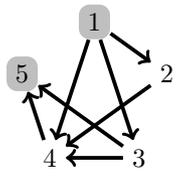
\end{document}